\documentclass{article}

\PassOptionsToPackage{numbers}{natbib}
\usepackage[final,main]{neurips_2026}
\usepackage[utf8]{inputenc} 
\usepackage[T1]{fontenc}    
\usepackage{hyperref}       
\usepackage{url}            
\usepackage{booktabs}       
\usepackage{nicefrac}       
\usepackage{microtype}      
\usepackage{amsmath, amssymb, amsthm}
\usepackage{mathtools}
\usepackage{multirow}
\usepackage{graphicx}
\usepackage{subcaption}
\usepackage{enumitem}
\usepackage[table]{xcolor}
\usepackage{bm}
\usepackage{algorithm}
\usepackage[noend]{algorithmic}
\usepackage{pifont}
\usepackage[textsize=tiny]{todonotes}
\usepackage{wrapfig}
\usepackage{float}
\usepackage{etoc}
\usepackage[capitalize,nameinlink]{cleveref}

\usepackage{minitoc} 
\usepackage[title,titletoc]{appendix}
\mtcsetfont{minitoc}{section}{\small\rmfamily}
\mtcsetfont{minitoc}{subsection}{\small\rmfamily}

\definecolor{dodgerblue}{RGB}{30, 144, 255}
\newcommand{\cmark}{\ding{51}}%
\newcommand{\xmark}{\ding{55}}%
\newtheorem{theorem}{Theorem}

\newtheorem{proposition}[theorem]{Proposition}

\newtheorem{assumption}{Assumption}
\newtheorem{remark}{Remark}

\title{Deconstructing Multi-Task Active Learning: \\
The Paradox of Gradient Conflict and Orthogonal Decomposition}

\author{%
  Juho Jung\thanks{Equal contribution.} \\
  Sungkyunkwan University \\
  \texttt{jhjeon9@g.skku.edu} \\
  \And
  Seongjae Kang\footnotemark[1] \\
  KAIST \\
  \texttt{tjdwo2744@kaist.ac.kr} \\
  \And
  Jinyoung Han\thanks{Corresponding author.} \\
  Sungkyunkwan University \\
  \texttt{jinyounghan@skku.edu} \\
}

\begin{document}
\etocdepthtag.toc{mainpaper}
\maketitle

\begin{abstract}
  Existing multi-task active learning (MTL-AL) methods predominantly rely on the linear scalarization of heterogeneous objectives, such as uncertainty, diversity, and gradient conflict. We argue that this entangled formulation is fundamentally limited, as linear combination fails to capture Pareto-optimal trade-offs under conflicting objectives. To address this limitation, we propose a paradigm shift from \emph{combination} to \emph{deconstruction} and introduce a principled framework that rethinks MTL-AL along three orthogonal dimensions. First, we identify the \textit{Conflict Paradox}: although the optimization literature treats gradient conflict as interference to be suppressed, carrying that intuition into the acquisition score is counterproductive --- penalizing conflict when ranking candidates is measurably worse than not scoring it at all. We therefore decouple the two concerns: the acquisition function does not score conflict, and the interference it causes is resolved separately by gradient surgery during training. Second, we introduce \textit{Orthogonal Decomposition}, separating the acquisition objective into two distinct axes: \textit{Informativeness} and \textit{Exploration}. This two-stage process, consisting of informativeness filtering followed by coverage maximization, prevents ``\textit{diversity dilution}'' by excluding uninformative outliers before enforcing diversity. Third, we propose \textit{Cross-Task Mutual Information (CTMI)}, a covariance-based statistic of cross-task dependency computed from MC-dropout task entropies. 
  Extensive experiments on COCO, Cityscapes, and NYUv2 --- including comparisons against recent multi-task active learning and data-selection baselines --- demonstrate consistent improvements over state-of-the-art MTL-AL methods, highlighting that structural rethinking of acquisition yields larger gains than incremental heuristic design. Our code is publicly available at: \href{https://github.com/Juho-Jung/multitask_active_learning_official}{\texttt{github.com/Juho-Jung/multitask\_active\_learning\_official}}
\end{abstract}

\section{Introduction}
\label{sec:intro}

Multi-task learning (MTL) has emerged as a fundamental paradigm for building efficient and generalizable representations by leveraging shared structures across related tasks~\citep{caruana1997multitask,ruder2017overview}. However, its success highly depends on the availability of densely annotated data for every task, resulting in prohibitive annotation costs as the number and complexity of tasks grow~\citep{zhang2021survey,standley2020tasks}. Active Learning (AL) offers an approach to mitigate this burden by selectively querying the most valuable samples~\citep{settles2009active}. Yet, extending AL to MTL settings introduces a fundamental challenge: the notion of ``\emph{informativeness}'' becomes inherently multi-faceted, requiring the joint consideration of uncertainty, diversity, and inter-task relationships~\citep{xiao2020efficient}.

\paragraph{The Fundamental Problem: Linear Combination.}

Most existing MTL active learning methods address this challenge by formulating sample selection as a multi-objective optimization problem solved via linear scalarization~\citep{xiao2020efficient,hekimoglu2024active,harpale2012multi,hekimoglu2023multi,reichart2008multi,rotman2022multi,chen2022active}. Representative approaches such as BADGE~\citep{ash2019deep} and BatchBALD~\citep{kirsch2019batchbald} combine uncertainty and diversity through weighted summation. In multi-task settings, this strategy is commonly extended by aggregating task-wise uncertainties~\citep{sinha2019variational,yoo2019learning} or incorporating penalties for gradient conflict~\citep{yu2020gradient,chen2018gradnorm}.
We argue this reliance on linear combination reflects a fundamental geometric limitation rather than a mere modeling choice. Since the objectives are incommensurable and measured in different units (e.g., entropy vs. feature distance vs. gradient cosine), linear scalarization fails to discover non-convex regions of the Pareto frontier~\citep{sener2018multi, lin2019pareto}, often leading to sub-optimal trade-offs driven by arbitrary weighting. This limitation manifests in two failure modes: ``\textit{Diversity Dilution},'' where outliers dominate due to extreme distance scores~\citep{sener2017active}, and the ``\textit{Conflict Paradox},'' where samples informative for inter-task trade-offs are prematurely discarded.

\paragraph{Deconstruction as a paradigm.} Rather than seeking better weighting schemes, we rethink the problem structure via two principles: (i) \textbf{Decoupling Selection from Optimization} — keep gradient conflict out of the acquisition score entirely, and resolve the interference it causes at training time via gradient surgery, instead of letting an optimization-time concern filter which samples are ever labelled; (ii) \textbf{Orthogonal Decomposition} — decompose acquisition into two sequential, conceptually orthogonal axes (\textit{Informativeness} and \textit{Exploration}) and apply them hierarchically so that diversity is enforced only within a subspace of genuinely informative candidates.


Our contributions are summarized as follows:
\begin{itemize}[leftmargin=1.2em,nosep]
    \item We identify the \textbf{Conflict Paradox}: penalizing gradient conflict inside the acquisition score is measurably worse than not scoring conflict at all, and we introduce a principled decoupling that leaves conflict to the optimizer rather than to the selector.
    \item We introduce \textbf{Orthogonal Decomposition}, a two-stage selection strategy that first filters candidates by informativeness and subsequently maximizes coverage within the retained set, preventing diversity dilution.
    \item We propose \textbf{Cross-Task Mutual Information (CTMI)}, a forward-pass-only statistic of cross-task dependency, and show that any non-zero cross-task covariance in our architecture is structurally attributable to the shared representation; we validate CTMI as an acquisition signal by ablation rather than by a distributional bound.
    \item We demonstrate consistent improvements on three benchmarks --- COCO (natural images), Cityscapes (autonomous driving) and NYUv2 (indoor scenes) --- reaching $+10.0\%$ over Random and $+2.4\%$ to $+3.0\%$ over prior active-learning methods under our main protocol (Table~\ref{tab:main_results}), and the highest paired gain against five recent multi-stage selection methods under a separate unified multi-seed protocol (Table~\ref{tab:unified_baselines}).
\end{itemize}

\section{Related Work}
\label{sec:related_work}

\paragraph{Active learning and the linear-combination trap.}
Deep AL is commonly categorized into uncertainty-based methods (entropy~\citep{shannon1948mathematical,gal2017deep}, margin~\citep{roth2006margin}, MC-Dropout~\citep{gal2016dropout}, deep ensembles~\citep{lakshminarayanan2017simple}), diversity-based methods (CoreSet~\citep{sener2017active}, discriminative coverage~\citep{gissin2019discriminative,sinha2019variational}), and hybrids that linearly combine the two: BADGE~\citep{ash2019deep} combines gradient magnitude with k-means++ clustering, BatchBALD~\citep{kirsch2019batchbald} maximizes batch-level MI, and VAAL~\citep{sinha2019variational}/Learning Loss~\citep{yoo2019learning} learn auxiliary scalar predictors. As we show in Section~\ref{sec:experiments} and Appendix~\ref{app:qualitative}, this linear-combination paradigm suffers from \emph{``Diversity Dilution''} — uninformative outliers with extreme distance scores dominate selection.

\begin{figure}[t]
    \begin{minipage}{0.5\textwidth}
        \centering
        \includegraphics[width=1.0\linewidth]{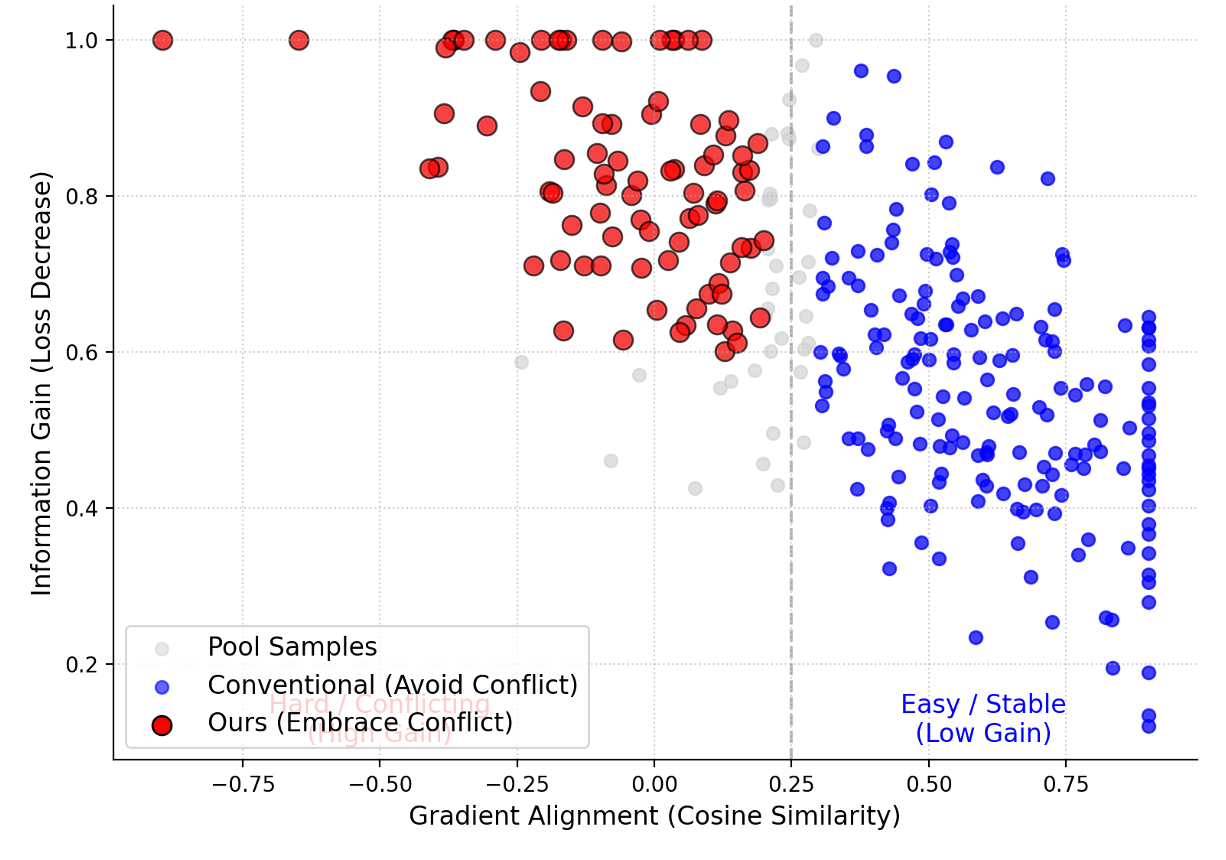}
    \end{minipage}\hfill
    \begin{minipage}{0.46\textwidth}
        \caption{\textbf{Visualizing the Conflict Paradox.} Retrospective oracle analysis on \textit{1,000 COCO benchmark} samples using a multi-task model pretrained on a \textit{1\% seed set}. Gradient alignment (cosine similarity, x-axis) is negatively correlated with actual loss reduction after training (y-axis) under AdamW with PCGrad off, ruling out the resolution mechanism as a confound: high-conflict samples yield larger information gain once basic features are learned.}
        \label{fig:paradox}
    \end{minipage}
\end{figure}

\paragraph{Multi-task learning, gradient conflict, and MTL active learning.} MTL suffers from gradient conflict between tasks, and every established remedy acts at \emph{optimization} time: task balancing (GradNorm~\citep{chen2018gradnorm}, MGDA~\citep{sener2018multi}) and gradient surgery (PCGrad~\citep{yu2020gradient}, CAGrad~\citep{liu2021conflict}). Cross-task conflict has, to our knowledge, never been used as an \emph{acquisition} signal: multi-task AL scores candidates by aggregated per-task uncertainty~\citep{reichart2008multi} or by extending single-task heuristics~\citep{shaikh2020examining,sharma2017learning,rotman2022multi}, leaving the selection stage blind to inter-task structure. We call the resulting situation the \emph{\textbf{Conflict Paradox}}: when a conflict term \emph{is} introduced into the acquisition score, penalizing it --- the reading that the optimization literature invites --- is measurably worse than omitting it altogether (Table~\ref{tab:ablation_conflict}). The closest recent methods, Hekimoglu et al.~\citep{hekimoglu2023multi,hekimoglu2024active} (PAD-Net inconsistency $+$ diversity) and PartAL~\citep{durasov2022partal} (per-task partial labeling), depend on specific task pairings or architectural components;
our CTMI is architecture-agnostic and uses only forward passes through the shared encoder. For mutual-information estimation, classical $k$-NN~\citep{kraskov2004estimating} and neural estimators (MINE~\citep{belghazi2018mutual}, CPC~\citep{oord2018representation}) are impractical inside the AL inner loop; CTMI instead requires only the empirical
covariance of MC-Dropout entropies, with no auxiliary network and no inner optimization.

\paragraph{Two-stage selection is established practice.}
Filtering a candidate pool by an informativeness score and then selecting for coverage within the filtered set is a well-established template. FASS~\citep{wei2015submodularity} filters by uncertainty before submodular subset selection; Cluster-Margin~\citep{citovsky2021batch} margin-filters before cluster-diverse sampling; PPAL~\citep{yang2024plug} applies difficulty-calibrated uncertainty sampling followed by category-conditioned diversity; $\mathbb{D}^2$-Pruning~\citep{maharana2024mathbb} and SES~\citep{xie2025structural} balance difficulty against diversity through graph message passing and structural entropy respectively; and related
submodular and coverage-based formulations are surveyed in~\citep{sener2017active,ash2019deep}. Our contribution is the multi-task-specific quantity that drives its first stage --- a cross-task dependency statistic computed from the shared encoder --- together with the analysis motivating it. Section~\ref{sec:experiments} isolates exactly this by running the same two-stage template with single-task signals under one identical multi-seed protocol.


\section{The Conflict Paradox}
\label{sec:conflict_paradox}

Existing multi-task active learning methods penalize gradient conflict, treating it as optimization noise~\citep{liu2021conflict,yu2020gradient,shi2023recon,chen2018gradnorm}.
As illustrated in Figure~\ref{fig:paradox}, this assumption induces a systemic bias toward the ``aligned'' region (blue area), where samples follow the model's current optimization trajectory. However, our retrospective analysis shows that after early training, samples yielding the largest joint-loss reduction concentrate in high-conflict regions (red area). Penalizing conflict thus discards high-gain samples. Motivated by this, we propose the \textbf{Decoupling Principle}, which explicitly targets this informative yet neglected region.

\subsection{Formalizing the Paradox}
Let $\mathcal{L}_t(x; \theta)$ be the loss for task $t$ on sample $x$. Gradient conflict is typically quantified by the cosine similarity:
\begin{equation}
    C_g(x) = \cos(\mathbf{g}_i, \mathbf{g}_j) = \frac{\mathbf{g}_i \cdot \mathbf{g}_j}{\|\mathbf{g}_i\| \|\mathbf{g}_j\|}, \quad \text{where } \mathbf{g}_t = \nabla_\theta \mathcal{L}_t(x)
\end{equation}


\begin{figure*}[t]
    \centering
    \includegraphics[width=1.0\textwidth]{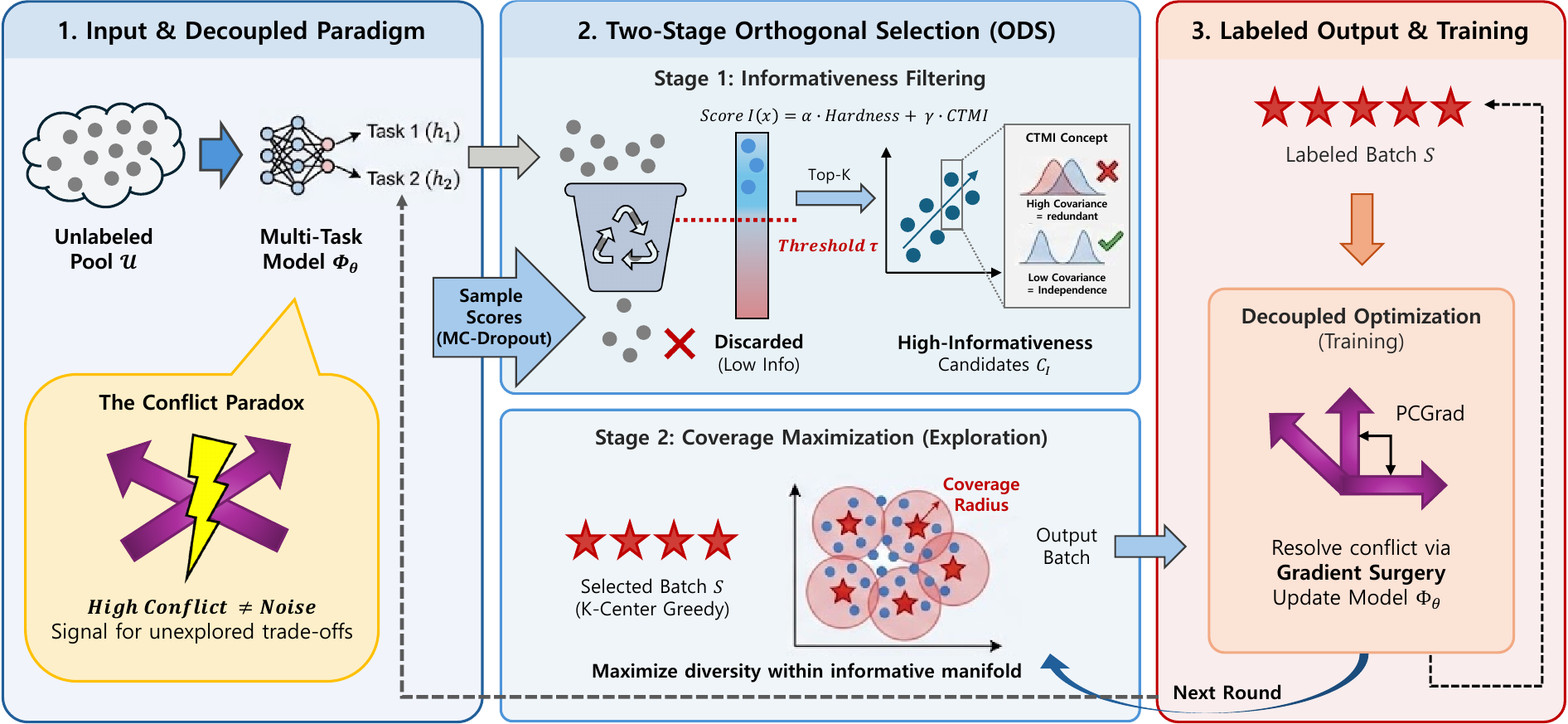}    
    \caption{\textbf{Overview of the Orthogonal Decomposition Selection (ODS) Framework.} 
    \textbf{(Left) The Conflict Paradox:} Conflict-aware selection treats gradient conflict as an informative signal.
    \textbf{(Middle) Orthogonal Decomposition:} Acquisition is decomposed into \textit{Informativeness Filtering} (Hardness and CTMI) followed by \textit{Coverage Maximization} within the retained candidates.
    \textbf{(Right) Decoupled Optimization:} Whatever gradient interference the selected batch contains is resolved at training time by \emph{PCGrad}.}
    \label{fig:ods_architecture}
\end{figure*}

\begin{proposition}[Sample-Level Conflict and Local Rank Deficiency]
\label{thm:conflict_paradox}
Let $J(x) = [\mathbf{g}_i(x), \mathbf{g}_j(x)] \in \mathbb{R}^{d \times 2}$ denote the per-sample task Jacobian. Under standard local-linearization assumptions on the shared encoder, samples with extreme gradient conflict ($C_g(x) \to -1$, i.e., $\mathbf{g}_i \approx -k \mathbf{g}_j$, $k>0$) satisfy:
\begin{enumerate}
    \item \textbf{Sample-level Pareto-critical structure.} There exists $\alpha \in (0,1)$ with $\alpha \mathbf{g}_i + (1-\alpha) \mathbf{g}_j = \mathbf{0}$, i.e., the per-sample loss vector $(\mathcal{L}_i(x), \mathcal{L}_j(x))$ admits no joint descent direction at the current $\theta$, and $J(x)$ is locally rank-deficient.
    \item \textbf{Bridge to active learning.} Acquiring such samples imposes localized constraints that the current shared representation cannot reconcile; resolving them via gradient projection (PCGrad~\citep{yu2020gradient}) forces the encoder to expand into directions orthogonal to the dominant trajectory, yielding feature-space coverage gain not attainable from low-conflict (cooperative rank-1) samples.
\end{enumerate}
\end{proposition}

\begin{remark}[Sample-level vs.\ objective-level Pareto stationarity]
\emph{Objective-level} Pareto stationarity is a global statement about $\theta$ over the data distribution and is strictly stronger than the sample-level structure above. We therefore deliberately phrase Proposition~\ref{thm:conflict_paradox} as a local, sample-conditional property — a structural heuristic that motivates our acquisition function — rather than a global optimization guarantee. The claim we draw from this is narrow and concerns the acquisition score only: with the optimizer held fixed, \emph{penalizing} conflict in the score (Row~A) is worse than \emph{omitting} it (Row~B), so the penalty removes candidates the model had something to learn from (Section~\ref{sec:experiments}).
\end{remark}

\begin{proof}[Proof Sketch]
In high-conflict regions where $\mathbf{g}_i \approx -k \mathbf{g}_j$ ($k>0$), the joint descent direction $\mathbf{g}_i+\mathbf{g}_j$ vanishes and $J(x)$ collapses to rank one but with antiparallel orientation (\emph{adversarial} rather than \emph{cooperative} rank-1). Standard SGD stalls at this point for sample $x$. Selecting and applying PCGrad breaks the deadlock by projecting gradients onto the normal plane, forcing the encoder to expand its feature space to satisfy the conflicting constraints. The formal statement, full assumptions, and complete derivation appear in Appendix~\ref{app:proof_paradox}.
\end{proof}

\subsection{The Decoupling Principle}


Based on Proposition~\ref{thm:conflict_paradox}, we resolve the tension between selection and optimization by acting on each separately: at \emph{selection} the acquisition score contains no conflict term at all, so conflicting samples are neither sought out nor filtered away; at \emph{optimization} we apply gradient surgery (PCGrad~\citep{yu2020gradient}) to project conflicting gradients onto the normal plane and stabilize updates without discarding the data. 

\section{Orthogonal Decomposition of Acquisition}
\label{sec:orthogonal_decomposition}

We propose \textbf{Orthogonal Decomposition Selection (ODS)}, a framework that departs from linear scalarization by decomposing acquisition into two structurally distinct axes: \textbf{Informativeness ($\mathcal{I}$)} and \textbf{Exploration ($\mathcal{E}$)}. ODS operates as a three-phase decoupled
pipeline (Figure~\ref{fig:ods_architecture}): \emph{(1) Scoring} — compute per-task predictive uncertainty and its cross-task structure under MC-dropout; \emph{(2) Two-Stage Selection} — filter the pool to the high-$\mathcal{I}$ candidates, then maximize coverage within them; \emph{(3) Decoupled Optimization} — train the selected batch under PCGrad~\citep{yu2020gradient}, which resolves whatever gradient conflict the batch happens to contain.

\paragraph{Notation.}
$\Phi_\theta(x) \in \mathbb{R}^d$ denotes the shared encoder representation of sample $x$; it is the feature used both for the cross-task statistic of Section~\ref{sec:ctmi} and for the Stage-2 coverage metric.
$\operatorname{Hardness}(x)$ denotes per-task predictive uncertainty aggregated over tasks: predictive entropy for classification, the top-$K$ box-entropy sum for detection, and the $90$th percentile of per-pixel entropy for segmentation. Both quantities are estimated from the same $K$ MC-dropout forward passes; full definitions are in Appendix~\ref{app:implementation}.

\subsection{Orthogonality of Objectives}
\label{sec:orthogonality_of_objecitves}


\begin{assumption}[Weak Dependence of $\mathcal{I}$ and $\mathcal{E}$]
\label{assum:orthogonality}
We do \emph{not} assume that \textit{Informativeness} $\mathcal{I}(x)$
(uncertainty magnitude and its cross-task structure) and \textit{Exploration}
$\mathcal{E}(x)$ (feature-space coverage) are statistically independent. We
assume only that their dependence is weak enough that neither is a useful
surrogate for the other:
\begin{equation}
    \big|\operatorname{Corr}\!\big(\mathcal{I}(x), \mathcal{E}(x)\big)\big| \ll 1 ,
    \qquad x \sim P .
\end{equation}
Under this condition a scalarized combination $\alpha\mathcal{I} +
\beta\mathcal{E}$ trades one axis against the other instead of reinforcing a
single underlying quantity, which is precisely the regime in which linear
scalarization is fragile.
\end{assumption}

\begin{wraptable}{r}{0.5\textwidth}
\vspace{-0.18in}
\centering
\caption{\textbf{Correlation between Informativeness $\mathcal{I}$ and Exploration $\mathcal{E}$ on COCO} (1{,}500 unlabeled samples). Both Pearson $r$ and Spearman $\rho$ stay near zero across the full 50-round trajectory and decay toward zero as the model converges, supporting Assumption~\ref{assum:orthogonality}.}
\label{tab:ie_correlation}
\footnotesize
\setlength{\tabcolsep}{3pt}
\resizebox{0.5\textwidth}{!}{%
    \begin{tabular}{lccccccc}
    \toprule
    Round & 1 & 5 & 10 & 20 & 30 & 40 & 49 \\
    \midrule
    Pearson $r$ & $-0.18$ & $-0.11$ & $-0.18$ & $-0.09$ & $-0.12$ & $-0.02$ & $-0.09$ \\
    Spearman $\rho$ & $-0.15$ & $-0.07$ & $-0.17$ & $-0.07$ & $-0.09$ & $+0.00$ & $-0.07$ \\
    \bottomrule
    \end{tabular}
    }
\end{wraptable}

\emph{Empirical support (COCO, $1{,}500$ unlabeled samples).} We measure $\operatorname{Corr}(\mathcal{I}, \mathcal{E})$ on saved checkpoints across the entire 50-round trajectory. As shown in Table~\ref{tab:ie_correlation}, both Pearson and Spearman correlations stay small in absolute value at every round ($|r| \le 0.18$); they fluctuate rather than decrease monotonically, and we make no claim that they decay. Because a near-zero correlation does not by itself rule out non-linear dependence, we additionally report distance correlation~\citep{szekely2007measuring}, which vanishes only under genuine
independence; it likewise remains small throughout (Appendix~\ref{app:ie_correlation}, Appendix~\ref{app:dcorr}). The persistent \emph{negative} sign in early and mid rounds is itself instructive:
$\mathcal{I}$ and $\mathcal{E}$ \emph{disagree}, so a linear scalarizer that pushes one up trades the other off --- exactly the diversity-dilution failure mode.

\emph{Justification.} This decomposition is a strategic design choice. Traditional linear combination ($\alpha \mathcal{I} + \beta \mathcal{E}$) creates a diagonal decision boundary that often leads to \textbf{Diversity Dilution}, where uninformative outliers dominate selection \citep{sener2017active}. By treating them orthogonally, ODS sequentially selects samples that satisfy \emph{both} criteria.


\begin{figure}[t]
    \centering
    \includegraphics[width=0.90\columnwidth]{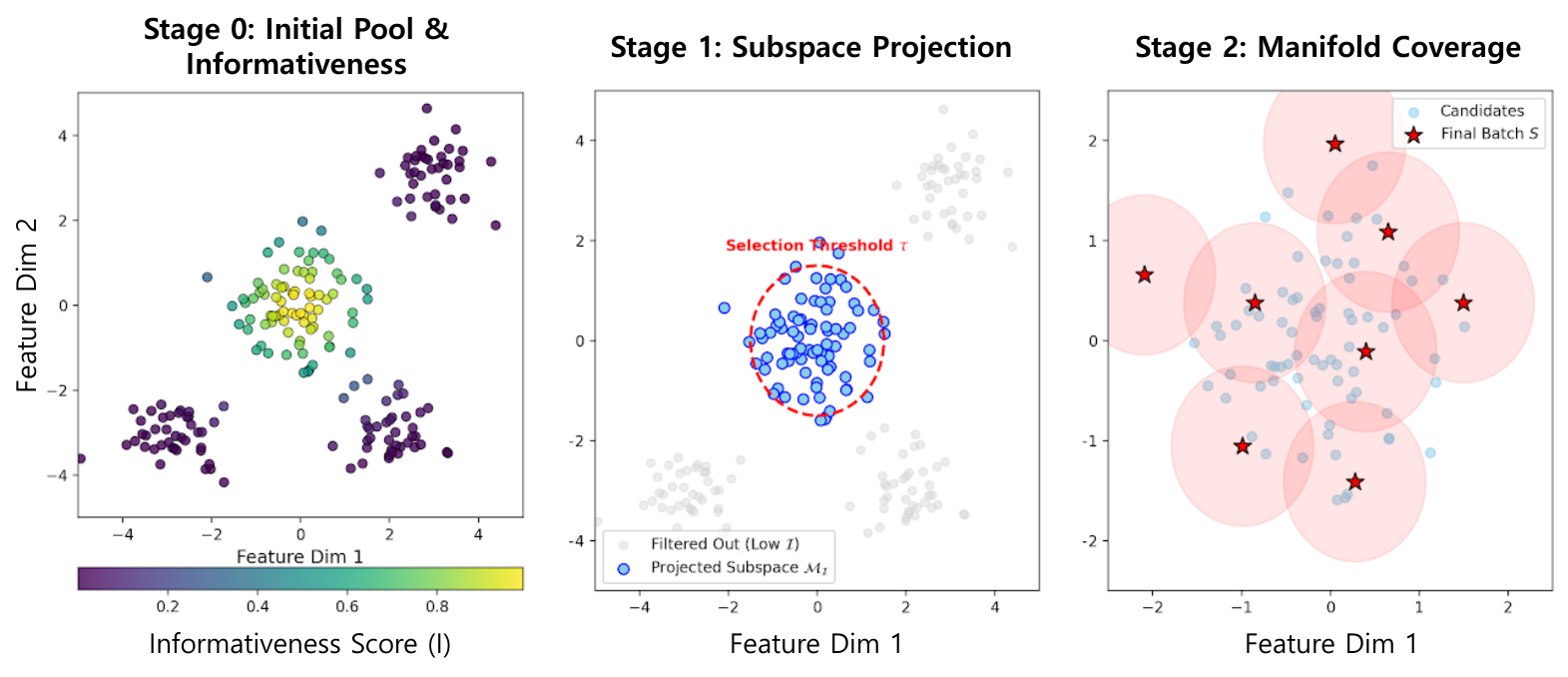}
    \caption{\textbf{The Orthogonal Decomposition Framework.} Instead of a single weighted sum, we process objectives sequentially. \textbf{Stage 1} projects data onto the High-Informativeness Subspace ($\mathcal{M}_{\mathcal{I}}$) (filtering out easy/redundant samples). \textbf{Stage 2} maximizes coverage within this subspace, preventing ``\emph{Diversity Dilution}.''}
    \label{fig:framework}
\end{figure}

\subsection{Two-Stage Orthogonal Selection (ODS)}
\label{sec:two_stage_ods}

To enforce Assumption~\ref{assum:orthogonality}, we operationalize the selection process into two sequential steps (Figure~\ref{fig:framework}).

\paragraph{Stage 1: Informativeness Filtering.}
We first restrict the candidate pool $\mathcal{U}$ to its high-informativeness subset $\mathcal{C}_{\mathcal{I}}$ by thresholding out ``\emph{easy}'' samples: $\mathcal{C}_{\mathcal{I}} = \{ x \in \mathcal{U} \mid \mathcal{I}(x) \ge \tau_k \}$ where $\tau_k$ corresponds to the top-$k \times m$ samples. Here, $\mathcal{I}(x)$ combines Hardness and our proposed \emph{CTMI} (Section~\ref{sec:ctmi}). This step explicitly removes uninformative outliers (gray points in Figure~\ref{fig:ods_architecture}), ensuring that diversity maximization focuses only on valuable regions.

\paragraph{Stage 2: Manifold Coverage (Maximizing $\mathcal{E}$).}
Within the projected subspace $\mathcal{M}_{\mathcal{I}}$, we solve for the batch $S$ that maximizes coverage using the \textit{k-Center Greedy} algorithm \citep{sener2017active}:
\begin{equation}
    \min_{S \subset \mathcal{C}_{\mathcal{I}},\, |S|=B}\ \max_{x \in \mathcal{C}_{\mathcal{I}}}\ \min_{s \in S}\ \big\| \Phi_\theta(x) - \Phi_\theta(s) \big\|_2
\end{equation}
Stage~1 shrinks the candidate set by $m\times$, so the cost is comparable to standard diversity-based methods. The full procedure is in Algorithm~\ref{alg:ods}.

\section{Cross-Task Mutual Information (CTMI)}
\label{sec:ctmi}


To quantify the \emph{Informativeness} ($\mathcal{I}$) of a candidate in Stage 1, we introduce CTMI. Per-task uncertainty alone is insufficient for this purpose. Two samples with the same mean entropy can have entirely different annotation value. If all tasks are uncertain due to a single shared bottleneck, their labels yield redundant supervisory signals. If they are uncertain for independent reasons, a single annotation resolves multiple distinct ambiguities. We therefore complement the uncertainty \emph{magnitude} (Hardness) with a statistic of its cross-task \emph{structure}, and prioritize samples whose task uncertainties are \emph{not} explained by a single shared cause.


\subsection{Definition}
\label{sec:ctmi_definition}

Let $\{\hat{y}^{(k)}(x)\}_{k=1}^{K}$ be the predictions of $K$ MC-dropout passes
and let $H_t^{(k)}(x)$ denote the entropy of task $t$ on pass $k$
(Appendix~\ref{app:entropy_regression} for continuous tasks). For each task pair
$(i,j) \in \mathcal{P}$, we compute the absolute correlation between the two
entropy sequences across passes, average over pairs, and take its complement:
\begin{equation}
\label{eq:ctmi}
\bar{\rho}(x) = \frac{1}{|\mathcal{P}|}\!\!\sum_{(i,j)\in\mathcal{P}}\!\!
\left|\frac{\operatorname{Cov}_k\!\big(H_i^{(k)}(x),H_j^{(k)}(x)\big)}
{\sqrt{\operatorname{Var}_k\!\big(H_i^{(k)}(x)\big)\operatorname{Var}_k\!\big(H_j^{(k)}(x)\big)}}\right|,
\qquad
\operatorname{CTMI}(x) = 1 - \widetilde{\bar{\rho}}(x),
\end{equation}
where $\widetilde{(\cdot)}$ denotes min--max normalization over the current
candidate pool. A sample attains a high CTMI score when its task uncertainties
fluctuate \emph{independently} under dropout, and a low score when they rise and
fall together. The complement is deliberate: tightly coupled task uncertainties
indicate that a single factor is responsible for all of them, so the $T$
annotations acquired for that sample resolve one ambiguity rather than $T$.

The Stage-1 informativeness score combines the two statistics with fixed
weights:
\begin{equation}
\label{eq:informativeness}
\mathcal{I}(x) = \alpha\cdot\operatorname{Hardness}(x) + \gamma\cdot\operatorname{CTMI}(x).
\end{equation}
Both terms are min--max normalized to $[0,1]$ prior to combination. Default
weights and the values used for each protocol are reported in
Appendix~\ref{app:implementation}.

\begin{table*}[t]
\centering
\caption{\textbf{Main Results on COCO Multi-Task Benchmark (Round 30, $3$ seeds: $\{42, 123, 456\}$).} 
All methods use the same shared MTL backbone (\textasciitilde2.40M params) and identical training protocol; only the acquisition function differs. We report mean$\pm$std on the constituent tasks, the \emph{Best} and \emph{Last} MTL Composite Score, and the \emph{Average Rank} across tasks. Note that \citet{hekimoglu2024active} additionally trains an acquisition-only autoencoder (\textasciitilde1.04M params per round); all other methods incur no extra trainable module.}
\label{tab:main_results}
\resizebox{\textwidth}{!}{%
    \begin{tabular}{l|ccc|ccc}
    \toprule
    \textbf{Method} & \textbf{Cls Acc (\%)} & \textbf{Det mAP} & \textbf{Seg mIoU} & \textbf{Best MTL Score ($\Delta m\%$)} $\uparrow$ & \textbf{Last MTL Score ($\Delta m\%$)} $\uparrow$ & \textbf{Avg Rank} $\downarrow$ \\
    \midrule
    \multicolumn{7}{l}{\textbf{\textit{Uncertainty-based}}} \\
    Random Sampling & 38.42 \scriptsize{$\pm$ 0.15} & 25.21 \scriptsize{$\pm$ 0.20} & 15.10 \scriptsize{$\pm$ 0.35} & 0.000 (-) & 0.000 (-) & 6.33 \\
    Entropy Sampling~\cite{shannon1948mathematical} & 38.30 \scriptsize{$\pm$ 0.12} & 25.13 \scriptsize{$\pm$ 0.18} & 14.97 \scriptsize{$\pm$ 0.22} & 0.000 (-) & -0.005 (-0.5\%) & 9.67 \\
    Margin & 38.25 \scriptsize{$\pm$ 0.15} & 25.16 \scriptsize{$\pm$ 0.21} & 15.06 \scriptsize{$\pm$ 0.25} & 0.000 (-) & -0.003 (-0.3\%) & 8.67 \\
    Learning Loss~\cite{yoo2019learning} & 38.36 \scriptsize{$\pm$ 0.15} & 25.18 \scriptsize{$\pm$ 0.21} & 15.05
 \scriptsize{$\pm$ 0.25} & 0.045 (+4.5\%) & 0.038 (+3.8\%) & 7.67 \\
    \midrule
    \multicolumn{7}{l}{\textbf{\textit{Diversity-based}}} \\
    CoreSet \cite{sener2017active} & 38.26 \scriptsize{$\pm$ 0.20} & 25.14 \scriptsize{$\pm$ 0.25} & 15.01 \scriptsize{$\pm$ 0.30} & 0.000 (-) & -0.002 (-0.2\%) & 9.33 \\
    BADGE \cite{ash2019deep} & 40.09 \scriptsize{$\pm$ 0.18} & 26.11 \scriptsize{$\pm$ 0.15} & 15.45 \scriptsize{$\pm$ 0.20} & 0.041 (+4.1\%) & 0.034 (+3.4\%) & 4.00 \\
    \midrule
    \multicolumn{7}{l}{\textbf{\textit{Hybrid / MTL-aware}}} \\
    BatchBALD \cite{kirsch2019batchbald} & 39.41 \scriptsize{$\pm$ 0.22} & 25.42 \scriptsize{$\pm$ 0.19} & 16.40 \scriptsize{$\pm$ 0.25} & 0.050 (+5.0\%) & 0.040 (+4.0\%) & 4.33 \\
    MT-BADGE (Ext.) & 41.22 \scriptsize{$\pm$ 0.15} & 25.81 \scriptsize{$\pm$ 0.12} & 16.83 \scriptsize{$\pm$ 0.18} & 0.075 (+7.5\%) & 0.070 (+7.0\%) & 2.67 \\
    PartAL~\citep{durasov2022partal} (adapted) & 38.43 \scriptsize{$\pm$ 0.13} & 20.62 \scriptsize{$\pm$ 0.11} & 16.34 \scriptsize{$\pm$ 0.02} & 0.049 (+4.9\%) & 0.038 (+3.8\%) & 7.33 \\
    Hekimoglu~\citep{hekimoglu2024active} (adapted) & 40.41 \scriptsize{$\pm$ 0.07} & 21.04 \scriptsize{$\pm$ 0.16} & 16.95 \scriptsize{$\pm$ 0.02} & 0.048 (+4.8\%) & 0.037 (+3.7\%) & 5.00 \\
    \midrule
    \rowcolor{gray!15} \textbf{ODS (Ours)} & \textbf{42.14 \scriptsize{$\pm$ 0.10}} & \textbf{26.52 \scriptsize{$\pm$ 0.15}} & \textbf{17.44 \scriptsize{$\pm$ 0.15}} & \textbf{0.105} (\textbf{+10.5\%}) & \textbf{0.100} (\textbf{+10.0\%}) & \textbf{1.00} \\
    \bottomrule
    \end{tabular}%
    }
\end{table*}

\subsection{What CTMI Measures}
\label{sec:ctmi_semantics}

$\bar\rho$ admits a \emph{cross-task} reading not because it bounds mutual
information, but by virtue of an attribution property of the estimator itself.

\begin{proposition}[Shared-Representation Attribution]
\label{prop:attribution}
Let $z = \Phi_\theta(x)$ be the shared representation and let the task
predictions be produced by task-private heads, $\hat{y}_t = f_t(z;\,\xi_t)$,
where the dropout masks $\xi_t$ are mutually independent given $z$. Then
$\operatorname{Cov}\!\big(H_i, H_j \mid z\big) = 0$ for $i \neq j$, and by the
law of total covariance
\begin{equation*}
\operatorname{Cov}(H_i, H_j) \;=\;
\underbrace{\mathbb{E}_z\!\big[\operatorname{Cov}(H_i,H_j \mid z)\big]}_{=\,0}
\;+\;\operatorname{Cov}_z\!\big(\mathbb{E}[H_i \mid z],\, \mathbb{E}[H_j \mid z]\big).
\end{equation*}
Every non-zero cross-task covariance is therefore attributable to variation in
the shared representation: head-private stochasticity alone cannot produce it.
(Proof in Appendix~\ref{app:proof_attribution}.)
\end{proposition}

Proposition~\ref{prop:attribution} identifies the origin of the signal: $\bar\rho$ isolates the component of a sample's joint uncertainty for which the shared encoder is responsible. This motivates the orientation of the complement in Eq.~\eqref{eq:ctmi}. A large $\bar\rho$ indicates that the shared encoder is the single bottleneck underlying every task's uncertainty on that sample, so the $T$ annotations resolve one common ambiguity. In contrast, a small $\bar\rho$ indicates that the tasks are uncertain for reasons the shared representation does not hold in common, and the same annotation budget yields several independent corrections. Selecting on $1 - \bar\rho$ is therefore the sample-level counterpart of the redundancy penalty that batch active learning applies \emph{between} samples~\citep{kirsch2019batchbald}; there, redundancy is avoided across images, whereas here it is avoided across tasks within a single image.

This orientation also prevents the two terms of $\mathcal{I}$ from measuring the same quantity. Uncertainty concentrated in a shared factor tends to be uncertainty that is simply large, so a score rewarding high $\bar\rho$ would largely restate Hardness and collapse Stage~1 onto a single axis, precisely the failure this paper sets out to avoid. Our empirical claim is restricted to the contribution of the term itself: under an identical Stage~2, adding CTMI to Hardness raises the MTL Composite from $0.094$ to $0.105$ (Appendix~\ref{app:component_ablation}), and this gain persists across combination rules (Section~\ref{sec:combination_rule}).




\subsection{On Combining Hardness and CTMI}
\label{sec:combination_rule}

Equation~\eqref{eq:informativeness} is a weighted sum, which may appear to sit uneasily beside our critique of linear scalarization. The two operate at different levels. Our Diversity-Dilution argument concerns sums across
\emph{heterogeneous axes}, namely model uncertainty against feature-space coverage, where one axis is measured in entropy and the other in feature distance, so that a single extreme coordinate can dominate and induce a
non-Pareto trade-off. Hardness and CTMI are instead two functionals of the \emph{same} MC-dropout predictive distribution, both normalized to $[0,1]$ over the same pool and both monotone in ``unresolved by the current model''. Their sum therefore acts as a soft conjunction of two screening conditions rather than a trade-off between competing objectives: no sample can rank highly on $\mathcal{I}$ without being uncertain.

While this weighted sum is ultimately a heuristic, its exact formulation matters little. Alternative combinations—such as a rank-sum, $z$-score sum, or a product—yield statistically indistinguishable results over three seeds ($0.723$--$0.731$). The acquisition gain therefore originates in the cross-task signal itself rather than the specific combination rule.
\section{Experiments}
\label{sec:experiments}


We validate our framework across three diverse benchmarks: \textbf{COCO} (general scenes), \textbf{Cityscapes} (autonomous driving), and \textbf{NYUv2} (indoor) to answer three questions:
\begin{enumerate}[leftmargin=*]
    \item \textbf{Performance:} Does ODS outperform state-of-the-art linear combination methods in data efficiency?
    \item \textbf{Paradox Verification:} Does embracing gradient conflict actually improve learning compared to avoiding it?
    \item \textbf{Theoretical Validity:} Does CTMI serve as a valid proxy for inter-task mutual information?
\end{enumerate}

\subsection{Experimental Setup}
\paragraph{Datasets.}
We consider datasets with varying task affinity (additional details are provided in Appendix~\ref{app:datasets}.):
\begin{itemize}[leftmargin=*]
    \item \textbf{COCO~\citep{lin2014microsoft}}: Object detection, instance segmentation, and image-level classification.
    \item \textbf{Cityscapes~\citep{cordts2016cityscapes}}: Semantic segmentation, depth estimation, and instance segmentation.
    \item \textbf{NYUv2~\citep{silberman2012indoor}}: Indoor semantic segmentation, depth, and surface normal estimation.
\end{itemize}



\paragraph{Baselines.}
We compare against representative methods spanning three categories: \textbf{(1) Uncertainty-based:} \emph{Random}, \emph{Entropy} (mean task entropy), \emph{Margin}, and \emph{Learning Loss}~\citep{yoo2019learning}, \textbf{(2) Diversity-based:} \emph{CoreSet}~\citep{sener2017active} (feature distance maximization), and \textbf{(3) Hybrid / MTL-Extended:} \emph{BADGE}~\citep{ash2019deep}, \emph{BatchBALD}~\citep{kirsch2019batchbald}, and \emph{MT-BADGE (Ext.)} (our extension of BADGE via gradient concatenation).
Note that we focus on architecture-agnostic and task-general active learning baselines to ensure fair comparison across diverse multi-task settings. Recent task-specific or architecture-dependent methods are excluded, as they do not directly apply to our unified evaluation protocol. We provide a detailed justification for including specific recent work in Appendix~\ref{app:baseline_justification}.

\paragraph{Implementation Details \& Fairness Control.}
We use lightweight, non-pretrained architectures (Appendix~\ref{app:implementation}) to isolate acquisition effects from pretraining. All methods share the same backbone, optimizer (PCGrad~\citep{yu2020gradient}), and AL protocol (1\% seed, 1\% budget per round, 30 rounds). For regression tasks (e.g., depth), uncertainty is the log-variance of MC-Dropout predictions (Appendix~\ref{app:entropy_regression}).


\paragraph{Evaluation Metric: MTL Composite Score.}
To handle heterogeneous metrics (e.g., Acc $\uparrow$, RMSE $\downarrow$), we report the \textbf{MTL Composite Score} ($\Delta m\%$), which measures the average relative improvement over Random Sampling:
\begin{equation}
    \Delta m\% = \frac{1}{T} \sum_{t=1}^{T} (-1)^{s_t} \frac{P_t^{\text{method}} - P_t^{\text{random}}}{P_t^{\text{random}}} \times 100
\end{equation}
where $s_t=1$ if the metric degrades with value (e.g., Error/RMSE), else $0$.
This formulation normalizes task-specific metrics by measuring relative improvement over a random sampling, ensuring that no single task dominates the composite score despite heterogeneous metric scales.

\subsection{Main Results}

\paragraph{Learning Efficiency and Robustness.}
Figure~\ref{fig:curves} shows the MTL Composite Score ($\Delta m\%$) on COCO across active learning rounds.
Specifically, our method (\emph{ODS}, red solid line) consistently outperforms all baselines, establishing an early lead and maintaining it throughout training. As shown qualitatively in Appendix~\ref{app:qualitative}, this early advantage arises from avoiding redundant and low-density outlier sampling in the initial rounds.
As confirmed in Table~\ref{tab:main_results}, \emph{ODS} achieves a final gain of \textbf{+10.0\%} over random sampling at Round 30, remaining the only method with sustained gains above $10\%$.
While the gap narrows as models converge (Rounds 20--30), \emph{ODS} still significantly outperforms the strongest baselines, \textit{MT-BADGE} (+7.0\%) and \textit{BatchBALD} (+4.0\%).
In contrast, \textit{BADGE} exhibits unstable convergence in MTL, dropping from a mid-training peak to a final gain of only +3.4\%.


\begin{figure}[t]
    \begin{minipage}{0.52\textwidth}
        \centering
        \includegraphics[width=0.90\linewidth]{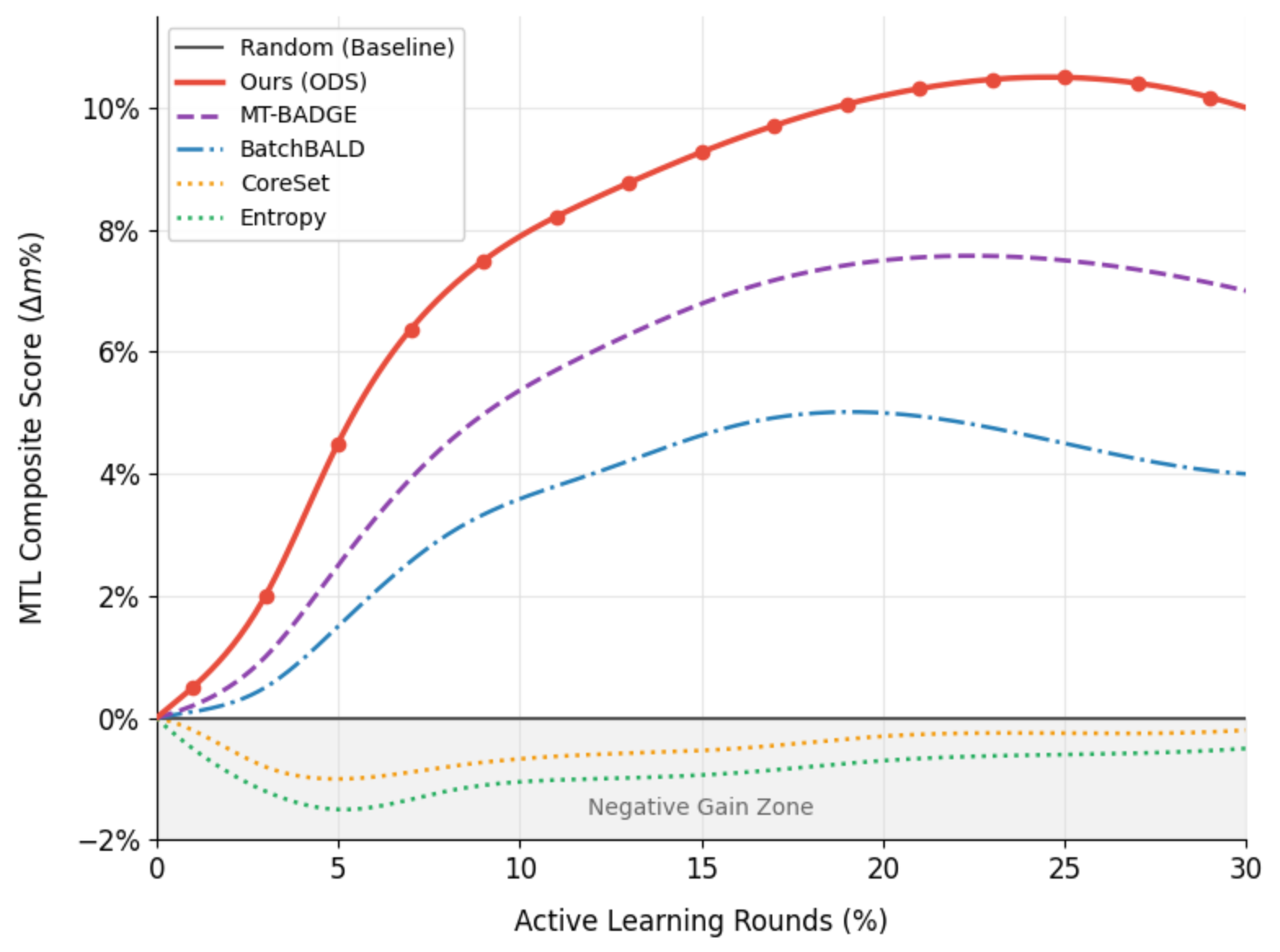}
    \end{minipage}\hfill
    \begin{minipage}{0.45\textwidth}
        \caption{\textbf{Evolution of MTL Composite Score ($\Delta m\%$) on COCO.} 
        The plot visualizes the relative improvement over the Random Sampling baseline ($y=0$). 
        Our method (ODS) consistently leads performance throughout training. 
        In contrast, standard baselines like \textit{Entropy} and \textit{CoreSet} fall into the negative gain zone ($y<0$), failing to outperform random sampling.}
        \label{fig:curves}
    \end{minipage}
\end{figure}

\paragraph{Evidence of Diversity Dilution.}
Uncertainty- and diversity-only baselines (Entropy, CoreSet) fall into the \emph{negative gain zone} in Figure~\ref{fig:curves}, underperforming random. Without task-aware filtering they prioritize visually distinctive but low-informativeness outliers, manifesting as both distributional collapse (t-SNE) and visual redundancy (Appendix~\ref{app:qualitative}). ODS removes such candidates in Stage~1 so that Stage~2 diversity is enforced only within the high-gain subspace (Appendix~\ref{app:case_study}).

\paragraph{Cross-Domain Robustness.}
ODS attains the largest gain on COCO (+10.0\% over Random) and also leads on Cityscapes (+8.8\%) and NYUv2 (+7.5\%), confirming that conflict-aware selection generalizes across object-centric and geometric multi-task settings. Per-task numbers (Cls/Det/Seg, Seg/Depth/Instance, Seg/Depth/Normal) are tabulated in Appendix~\ref{app:robustness_detail}.

\paragraph{Comparison with recent multi-stage selection methods.}
The two-stage template we use is shared with a line of recent active-learning and data-selection work (Section~\ref{sec:related_work}). To test whether the gain comes from the template or from the cross-task signal that drives its first stage, we adapt five such methods to our multi-task setting and run all of them under a single protocol: identical backbone, identical training schedule, three seeds, and $\Delta m$ computed \emph{per seed} against the same-seed Random run. Adaptation choices are documented in Appendix~\ref{app:baseline_adapt}.

\begin{table}[ht]
\centering
\caption{\textbf{Adapted multi-stage baselines under a unified protocol} (COCO, seeds $\{42,123,456\}$, 31 rounds). $\Delta m$ is measured per seed against the same-seed Random run; \emph{Best} is the peak over rounds and \emph{Trajectory} the mean over all rounds (area under the label-efficiency curve). \textbf{This protocol differs from Table~\ref{tab:main_results}}, which reports the original single-protocol comparison against classical baselines; the two sets of numbers are not interchangeable.}
\label{tab:unified_baselines}
\small
\setlength{\tabcolsep}{5pt}
\begin{tabular}{l|ccc|c|c}
\toprule
\textbf{Method} & \multicolumn{3}{c|}{\textbf{Best $\Delta m$ (\%) per seed}} & \textbf{Best $\Delta m$ (\%)} & \textbf{Trajectory $\Delta m$ (\%)} \\
 & 42 & 123 & 456 & mean $\pm$ std & mean $\pm$ std \\
\midrule
FASS~\citep{wei2015submodularity} (adapted)            & $+6.00$ & $+3.17$ & $+2.75$ & $3.97 \pm 1.44$ & $-0.15 \pm 0.12$ \\
Cluster-Margin~\citep{citovsky2021batch} (adapted)     & $+5.70$ & $+2.27$ & $+2.50$ & $3.49 \pm 1.57$ & $+0.32 \pm 0.10$ \\
PPAL~\citep{yang2024plug} (adapted)                    & $+5.98$ & $+1.68$ & $+1.11$ & $2.92 \pm 2.17$ & $+0.08 \pm 0.52$ \\
$\mathbb{D}^2$-Pruning~\citep{maharana2024mathbb} (adapted) & $+7.12$ & $+1.55$ & $+1.64$ & $3.44 \pm 2.61$ & $+0.09 \pm 0.02$ \\
SES~\citep{xie2025structural} (adapted)                & $+7.42$ & $+0.91$ & $+4.12$ & $4.15 \pm 2.66$ & $+0.41 \pm 0.62$ \\
\midrule
\rowcolor{gray!15}\textbf{ODS (Ours)} & $\mathbf{+7.35}$ & $\mathbf{+2.34}$ & $\mathbf{+3.52}$ & $\mathbf{4.41 \pm 2.14}$ & $\mathbf{+1.03 \pm 0.56}$ \\
\bottomrule
\end{tabular}
\end{table}

Read as unpaired means, the six methods overlap within one standard deviation. 
At peak, ODS gives the larger gain in $12$ of the $15$ paired (seed, baseline) comparisons; over the whole trajectory, which uses every round rather than a single peak, it leads in $13$ of $15$ and is the only method positive on all three seeds. The adapted single-task methods sit between $-0.15$ and $+0.41$ on the trajectory mean, i.e.\ close to Random, while ODS reaches $+1.03$. 

\subsection{Analytic Experiments}

\begin{table}[ht]
    \centering
    \caption{\textbf{Verifying the Conflict Paradox.} Comparing conflict handling strategies. \textbf{(A)} Penalizing conflict (conventional wisdom) yields the worst performance. \textbf{(B)} Simply ignoring it is better. \textbf{(C)} Embracing high-conflict samples without Stage~2 already surpasses both. \textbf{(D)} The full ODS pipeline (embrace + Stage~2 K-Center) achieves SOTA, isolating the contribution of structural diversity on top of conflict-aware selection.}
    \label{tab:ablation_conflict}
    \resizebox{0.85\columnwidth}{!}{%
    \begin{tabular}{c|l|c|c|c}
    \toprule
    \textbf{ID} & \textbf{Strategy} & \textbf{Selection} & \textbf{Optimization} & \textbf{MTL Comp.} \\
    \midrule
    (A) & \textbf{Penalize Conflict} & Unified ($1-|C_g|$) & AdamW & 0.0589 \\
    (B) & \textbf{Ignore Conflict} & Unified (No $C_g$) & AdamW & 0.0756 \\
    (C) & \textbf{Embrace (No Stage 2)} & Embrace & PCGrad & 0.0932 \\
    \midrule
    \rowcolor{gray!15} (D) & \textbf{Ours (Full)} & \textbf{Embrace + 2-Stage} & \textbf{PCGrad} & \textbf{0.105} \\
    \bottomrule
    \end{tabular}%
    }
\end{table}

\paragraph{Decoupling selection from optimization.}
Swapping the optimizer (AdamW vs.\ PCGrad) under a fixed acquisition isolates the selection effect from the optimization effect. With AdamW, ODS attains $+8.2\%$ vs.\ MT-BADGE's $+5.2\%$; with PCGrad, ODS attains $+10.5\%$ vs.\ $+7.5\%$ (full table in Appendix~\ref{app:pcgrad_ablation}). Hence ODS's advantage holds under both optimizers: PCGrad governs \emph{how} to learn, while ODS governs \emph{what} to learn.

\paragraph{Verifying the Conflict Paradox and ODS.}
Table~\ref{tab:ablation_conflict} compares three conflict-handling strategies: \textbf{Penalize} ($1-|C_g|$), \textbf{Ignore} (uncertainty only), and \textbf{Embrace} (high-conflict + PCGrad). Penalizing conflict yields the \emph{worst} score (0.059), even below ignoring it (0.076), directly supporting the Conflict Paradox: penalization systematically discards informative samples. Embracing without diversity (Row~C) reaches 0.093, while the full ODS (Row~D, embrace + Stage~2 K-Center) achieves 0.105, confirming that the Stage~2 coverage step removes redundancy among high-conflict candidates that tend to cluster in feature space.

\paragraph{Linear scalarization vs.\ orthogonal decomposition --- isolating the combination effect.}
A natural concern is that the gain may simply come from combining strong signals (Hardness, CTMI, Diversity) rather than from the decoupling itself. We run a head-to-head comparison using the \emph{exact same} per-sample scores, differing only in combination: linear scalarization $s = (1-\delta)(\alpha H + \gamma\,\mathrm{CTMI}) + \delta\,\mathrm{Div}$ vs.\ our two-stage decomposition. On COCO (seed~42, 30 rounds), ODS outperforms linear scalarization by $+6.5\%$ on the MTL Composite, with the largest gap on segmentation ($+6.2$~points; full table in Appendix~\ref{app:linear_vs_ods}). This empirically isolates the contribution of the \emph{decomposition itself}.

\section{Limitations}
\label{sec:limitations}

Our analysis has four explicit limitations. Proposition~\ref{thm:conflict_paradox} is a sample-level, local property of the per-sample Jacobian rather than a global Pareto-stationarity claim, making the bridge to AL utility empirical (Table~\ref{tab:ablation_conflict}). CTMI relies on covariance and is blind to non-linear dependencies; however, distance correlation tracks Pearson within $0.05$ on COCO (Appendix~\ref{app:dcorr}), indicating task uncertainty dependence is predominantly linear here. Domains with strongly non-linear coupling may require different statistics. ODS introduces hyperparameters $(m, K, \alpha, \gamma)$; while robust (Appendices~\ref{app:sensitivity}, \ref{app:alpha_gamma}), per-domain tuning may help. Finally, main results use lightweight scratch encoders; a pretrained ResNet-50 (Appendix~\ref{app:pretrained_ms}) preserves rank ordering but narrows gaps, leaving comprehensive evaluations with modern foundation backbones to future work.

\section{Conclusion}
\label{sec:conclusion}

This work replaces the standard linear combination paradigm in multi-task active learning with objective deconstruction. We identify the \textit{Conflict Paradox}, showing that importing the optimizer's aversion to gradient conflict into the acquisition score degrades accuracy rather than protecting it, and resolve this by decoupling sample selection from optimization. By separating informativeness from exploration, our \textbf{Orthogonal Decomposition} framework—driven by \textbf{Cross-Task Mutual Information (CTMI)}—prevents the diversity dilution that plagues scalarized objectives. Results across COCO, Cityscapes, and NYUv2 show consistent state-of-the-art performance, demonstrating that progress in MTL-AL lies in exploiting objective orthogonality rather than tuning heuristic scalarizations.

\section*{Acknowledgements}
\vspace{-5pt}
This work was supported by the National Research Foundation of Korea(NRF) grant funded by the Korea government(MSIT) (No. RS-2023-NR077194) and the MSIT(Ministry of Science, ICT), Korea, under the Global Scholars Invitation Program(RS-2024-00459638) supervised by the IITP(Institute for Information \& Communications Technology Planning \& Evaluation).

\bibliography{ref}

@article{caruana1997multitask,
  title={Multitask learning},
  author={Caruana, Rich},
  journal={Machine learning},
  volume={28},
  number={1},
  pages={41--75},
  year={1997},
  publisher={Springer}
}

@article{ruder2017overview,
  title={An overview of multi-task learning in deep neural networks},
  author={Ruder, Sebastian},
  journal={arXiv preprint arXiv:1706.05098},
  year={2017}
}

@article{zhang2021survey,
  title={A survey on multi-task learning},
  author={Zhang, Yu and Yang, Qiang},
  journal={IEEE transactions on knowledge and data engineering},
  volume={34},
  number={12},
  pages={5586--5609},
  year={2021},
  publisher={IEEE}
}

@article{settles2009active,
  title={Active learning literature survey},
  author={Settles, Burr},
  year={2009},
  publisher={University of Wisconsin-Madison Department of Computer Sciences}
}

@article{xiao2020efficient,
  title={An efficient active learning method for multi-task learning},
  author={Xiao, Yanshan and Chang, Zheng and Liu, Bo},
  journal={Knowledge-Based Systems},
  volume={190},
  pages={105137},
  year={2020},
  publisher={Elsevier}
}

@inproceedings{hekimoglu2024active,
  title={Active learning with task consistency and diversity in multi-task networks},
  author={Hekimoglu, Aral and Schmidt, Michael and Marcos-Ramiro, Alvaro},
  booktitle={2024 IEEE/CVF Winter Conference on Applications of Computer Vision (WACV)},
  pages={2491--2500},
  year={2024},
  organization={IEEE}
}

@inproceedings{hekimoglu2023multi,
  title={Multi-task consistency for active learning},
  author={Hekimoglu, Aral and Friedrich, Philipp and Zimmer, Walter and Schmidt, Michael and Marcos-Ramiro, Alvaro and Knoll, Alois},
  booktitle={2023 IEEE/CVF International Conference on Computer Vision Workshops (ICCVW)},
  pages={3407--3416},
  year={2023},
  organization={IEEE}
}

@phdthesis{harpale2012multi,
  title={Multi-task active learning},
  author={Harpale, Abhay},
  year={2012},
  school={Carnegie Mellon University}
}

@article{yu2020gradient,
  title={Gradient surgery for multi-task learning},
  author={Yu, Tianhe and Kumar, Saurabh and Gupta, Abhishek and Levine, Sergey and Hausman, Karol and Finn, Chelsea},
  journal={Advances in neural information processing systems},
  volume={33},
  pages={5824--5836},
  year={2020}
}

@inproceedings{chen2018gradnorm,
  title={Gradnorm: Gradient normalization for adaptive loss balancing in deep multitask networks},
  author={Chen, Zhao and Badrinarayanan, Vijay and Lee, Chen-Yu and Rabinovich, Andrew},
  booktitle={International conference on machine learning},
  pages={794--803},
  year={2018},
  organization={PMLR}
}

@article{sener2018multi,
  title={Multi-task learning as multi-objective optimization},
  author={Sener, Ozan and Koltun, Vladlen},
  journal={Advances in neural information processing systems},
  volume={31},
  year={2018}
}

@article{ash2019deep,
  title={Deep batch active learning by diverse, uncertain gradient lower bounds},
  author={Ash, Jordan T and Zhang, Chicheng and Krishnamurthy, Akshay and Langford, John and Agarwal, Alekh},
  journal={arXiv preprint arXiv:1906.03671},
  year={2019}
}

@article{kirsch2019batchbald,
  title={Batchbald: Efficient and diverse batch acquisition for deep bayesian active learning},
  author={Kirsch, Andreas and Van Amersfoort, Joost and Gal, Yarin},
  journal={Advances in neural information processing systems},
  volume={32},
  year={2019}
}

@inproceedings{sinha2019variational,
  title={Variational adversarial active learning},
  author={Sinha, Samarth and Ebrahimi, Sayna and Darrell, Trevor},
  booktitle={Proceedings of the IEEE/CVF international conference on computer vision},
  pages={5972--5981},
  year={2019}
}

@inproceedings{yoo2019learning,
  title={Learning loss for active learning},
  author={Yoo, Donggeun and Kweon, In So},
  booktitle={2019 IEEE/CVF Conference on Computer Vision and Pattern Recognition (CVPR)},
  pages={93--102},
  year={2019},
  organization={IEEE}
}

@article{lin2019pareto,
  title={Pareto multi-task learning},
  author={Lin, Xi and Zhen, Hui-Ling and Li, Zhenhua and Zhang, Qing-Fu and Kwong, Sam},
  journal={Advances in neural information processing systems},
  volume={32},
  year={2019}
}

@article{sener2017active,
  title={Active learning for convolutional neural networks: A core-set approach},
  author={Sener, Ozan and Savarese, Silvio},
  journal={arXiv preprint arXiv:1708.00489},
  year={2017}
}

@article{shannon1948mathematical,
  title={A mathematical theory of communication},
  author={Shannon, Claude Elwood},
  journal={The Bell system technical journal},
  volume={27},
  number={3},
  pages={379--423},
  year={1948},
  publisher={Nokia Bell Labs}
}

@inproceedings{gal2017deep,
  title={Deep bayesian active learning with image data},
  author={Gal, Yarin and Islam, Riashat and Ghahramani, Zoubin},
  booktitle={International conference on machine learning},
  pages={1183--1192},
  year={2017},
  organization={PMLR}
}

@inproceedings{roth2006margin,
  title={Margin-based active learning for structured output spaces},
  author={Roth, Dan and Small, Kevin},
  booktitle={European conference on machine learning},
  pages={413--424},
  year={2006},
  organization={Springer}
}

@inproceedings{gal2016dropout,
  title={Dropout as a bayesian approximation: Representing model uncertainty in deep learning},
  author={Gal, Yarin and Ghahramani, Zoubin},
  booktitle={international conference on machine learning},
  pages={1050--1059},
  year={2016},
  organization={PMLR}
}

@article{lakshminarayanan2017simple,
  title={Simple and scalable predictive uncertainty estimation using deep ensembles},
  author={Lakshminarayanan, Balaji and Pritzel, Alexander and Blundell, Charles},
  journal={Advances in neural information processing systems},
  volume={30},
  year={2017}
}

@article{gissin2019discriminative,
  title={Discriminative active learning},
  author={Gissin, Daniel and Shalev-Shwartz, Shai},
  journal={arXiv preprint arXiv:1907.06347},
  year={2019}
}

@inproceedings{shaikh2020examining,
  title={Examining the ordering of rhetorical strategies in persuasive requests},
  author={Shaikh, Omar and Chen, Jiaao and Saad-Falcon, Jon and Chau, Polo and Yang, Diyi},
  booktitle={Findings of the Association for Computational Linguistics: EMNLP 2020},
  pages={1299--1306},
  year={2020}
}

@article{sharma2017learning,
  title={Learning to multi-task by active sampling},
  author={Sharma, Sahil and Jha, Ashutosh and Hegde, Parikshit and Ravindran, Balaraman},
  journal={arXiv preprint arXiv:1702.06053},
  year={2017}
}

@article{kraskov2004estimating,
  title={Estimating mutual information},
  author={Kraskov, Alexander and St{\"o}gbauer, Harald and Grassberger, Peter},
  journal={Physical Review E—Statistical, Nonlinear, and Soft Matter Physics},
  volume={69},
  number={6},
  pages={066138},
  year={2004},
  publisher={APS}
}

@inproceedings{belghazi2018mutual,
  title={Mutual information neural estimation},
  author={Belghazi, Mohamed Ishmael and Baratin, Aristide and Rajeshwar, Sai and Ozair, Sherjil and Bengio, Yoshua and Courville, Aaron and Hjelm, Devon},
  booktitle={International conference on machine learning},
  pages={531--540},
  year={2018},
  organization={PMLR}
}

@article{oord2018representation,
  title={Representation learning with contrastive predictive coding},
  author={Oord, Aaron van den and Li, Yazhe and Vinyals, Oriol},
  journal={arXiv preprint arXiv:1807.03748},
  year={2018}
}

@inproceedings{reichart2008multi,
  title={Multi-task active learning for linguistic annotations},
  author={Reichart, Roi and Tomanek, Katrin and Hahn, Udo and Rappoport, Ari},
  booktitle={Proceedings of ACL-08: HLT},
  pages={861--869},
  year={2008}
}

@article{liu2021conflict,
  title={Conflict-averse gradient descent for multi-task learning},
  author={Liu, Bo and Liu, Xingchao and Jin, Xiaojie and Stone, Peter and Liu, Qiang},
  journal={Advances in neural information processing systems},
  volume={34},
  pages={18878--18890},
  year={2021}
}

@article{shi2023recon,
  title={Recon: Reducing conflicting gradients from the root for multi-task learning},
  author={Shi, Guangyuan and Li, Qimai and Zhang, Wenlong and Chen, Jiaxin and Wu, Xiao-Ming},
  journal={arXiv preprint arXiv:2302.11289},
  year={2023}
}

@inproceedings{lin2014microsoft,
  title={Microsoft coco: Common objects in context},
  author={Lin, Tsung-Yi and Maire, Michael and Belongie, Serge and Hays, James and Perona, Pietro and Ramanan, Deva and Doll{\'a}r, Piotr and Zitnick, C Lawrence},
  booktitle={European conference on computer vision},
  pages={740--755},
  year={2014},
  organization={Springer}
}

@inproceedings{cordts2016cityscapes,
  title={The cityscapes dataset for semantic urban scene understanding},
  author={Cordts, Marius and Omran, Mohamed and Ramos, Sebastian and Rehfeld, Timo and Enzweiler, Markus and Benenson, Rodrigo and Franke, Uwe and Roth, Stefan and Schiele, Bernt},
  booktitle={Proceedings of the IEEE conference on computer vision and pattern recognition},
  pages={3213--3223},
  year={2016}
}

@inproceedings{silberman2012indoor,
  title={Indoor segmentation and support inference from rgbd images},
  author={Silberman, Nathan and Hoiem, Derek and Kohli, Pushmeet and Fergus, Rob},
  booktitle={European conference on computer vision},
  pages={746--760},
  year={2012},
  organization={Springer}
}

@article{brust2018active,
  title={Active learning for deep object detection},
  author={Brust, Clemens-Alexander and K{\"a}ding, Christoph and Denzler, Joachim},
  journal={arXiv preprint arXiv:1809.09875},
  year={2018}
}

@article{citovsky2021batch,
  title={Batch active learning at scale},
  author={Citovsky, Gui and DeSalvo, Giulia and Gentile, Claudio and Karydas, Lazaros and Rajagopalan, Anand and Rostamizadeh, Afshin and Kumar, Sanjiv},
  journal={Advances in Neural Information Processing Systems},
  volume={34},
  pages={11933--11944},
  year={2021}
}

@inproceedings{standley2020tasks,
  title={Which tasks should be learned together in multi-task learning?},
  author={Standley, Trevor and Zamir, Amir and Chen, Dawn and Guibas, Leonidas and Malik, Jitendra and Savarese, Silvio},
  booktitle={International conference on machine learning},
  pages={9120--9132},
  year={2020},
  organization={PMLR}
}

@article{arefeen2025cost,
  title={Cost-Effective Multitask Active Learning in Wearable Sensor Systems},
  author={Arefeen, Asiful and Ghasemzadeh, Hassan},
  journal={Sensors (Basel, Switzerland)},
  volume={25},
  number={5},
  pages={1522},
  year={2025}
}

@article{qin2023infobatch,
  title={Infobatch: Lossless training speed up by unbiased dynamic data pruning},
  author={Qin, Ziheng and Wang, Kai and Zheng, Zangwei and Gu, Jianyang and Peng, Xiangyu and Xu, Zhaopan and Zhou, Daquan and Shang, Lei and Sun, Baigui and Xie, Xuansong and others},
  journal={arXiv preprint arXiv:2303.04947},
  year={2023}
}

@article{xia2023reliable,
  title={Reliable active learning via influence functions},
  author={Xia, Meng and Henao, Ricardo},
  journal={Transactions on Machine Learning Research},
  year={2023}
}

@article{gashi2025deep,
  title={Deep active learning: A reality check},
  author={Gashi, Edrina and Deng, Jiankang and Elezi, Ismail},
  journal={Pattern Recognition Letters},
  volume={197},
  pages={188--194},
  year={2025},
  publisher={Elsevier}
}

@article{patapati2025graph,
  title={Graph Coloring for Multi-Task Learning},
  author={Patapati, Santosh and Noronha, Ian},
  journal={arXiv preprint arXiv:2509.16959},
  year={2025}
}

@article{rotman2022multi,
  title={Multi-task active learning for pre-trained transformer-based models},
  author={Rotman, Guy and Reichart, Roi},
  journal={Transactions of the Association for Computational Linguistics},
  volume={10},
  pages={1209--1228},
  year={2022},
  publisher={MIT Press One Broadway, 12th Floor, Cambridge, Massachusetts 02142, USA~…}
}

@inproceedings{chen2022active,
  title={Active multi-task representation learning},
  author={Chen, Yifang and Jamieson, Kevin and Du, Simon},
  booktitle={International Conference on Machine Learning},
  pages={3271--3298},
  year={2022},
  organization={PMLR}
}

@incollection{kuhn2013nonlinear,
  title={Nonlinear programming},
  author={Kuhn, Harold W and Tucker, Albert W},
  booktitle={Traces and emergence of nonlinear programming},
  pages={247--258},
  year={2013},
  publisher={Springer}
}

@article{durasov2022partal,
  title={Partal: Efficient partial active learning in multi-task visual settings},
  author={Durasov, Nikita and Dorndorf, Nik and Fua, Pascal},
  journal={arXiv preprint arXiv:2211.11546},
  year={2022}
}

@article{szekely2007measuring,
  title={Measuring and testing dependence by correlation of distances},
  author={Sz{\'e}kely, G{\'a}bor J. and Rizzo, Maria L. and Bakirov, Nail K.},
  journal={The Annals of Statistics},
  volume={35},
  number={6},
  pages={2769--2794},
  year={2007}
}

@inproceedings{smith2024making,
  title={Making Better Use of Unlabelled Data in {B}ayesian Active Learning},
  author={Smith, Freddie Bickford and Foster, Adam and Rainforth, Tom},
  booktitle={International Conference on Artificial Intelligence and Statistics},
  year={2024}
}

@inproceedings{wei2015submodularity,
  title={Submodularity in Data Subset Selection and Active Learning},
  author={Wei, Kai and Iyer, Rishabh and Bilmes, Jeff},
  booktitle={Proceedings of the 32nd International Conference on Machine Learning},
  pages={1954--1963},
  year={2015}
}

@inproceedings{yang2024plug,
  title={Plug and Play Active Learning for Object Detection},
  author={Yang, Chenhongyi and Huang, Lichao and Crowley, Elliot J.},
  booktitle={Proceedings of the IEEE/CVF Conference on Computer Vision and Pattern Recognition},
  pages={17784--17793},
  year={2024}
}

@inproceedings{maharana2024mathbb,
  title={{$\mathbb{D}^2$} Pruning: Message Passing for Balancing Diversity \& Difficulty in Data Pruning},
  author={Maharana, Adyasha and Yadav, Prateek and Bansal, Mohit},
  booktitle={International Conference on Learning Representations},
  volume={2024},
  pages={55740--55759},
  year={2024}
}

@inproceedings{xie2025structural,
  title={Structural-Entropy-Based Sample Selection for Efficient and Effective Learning},
  author={Xie, Tianchi and Zhu, Jiangning and Ma, Guozu and Lin, Minzhi and Chen, Wei and Yang, Weikai and Liu, Shixia},
  booktitle={International Conference on Learning Representations},
  year={2025}
}
\bibliographystyle{plainnat}

\newpage
\appendix

\begin{center}
  {\huge\bfseries Appendix}
\end{center}
\vspace{0.5em}

\etocsetnexttocdepth{subsection}

\etocsetstyle{section}
  {}{}
  {\noindent\makebox[2em][l]{\textbf{\etocnumber}}%
   \textbf{\etocname}\dotfill\textbf{\etocpage}\par\addvspace{6pt}}
  {}

\etocsetstyle{subsection}
  {}{}
  {\noindent\hspace*{2em}%
   \makebox[2.5em][l]{\etocnumber}%
   \etocname\dotfill\etocpage\par\addvspace{2pt}}
  {}

\etocsettocstyle{{\large\bfseries Contents}\par\vspace{0.3em}}{\vspace{1em}}

\etocsettagdepth{mainpaper}{none}
\etocsettagdepth{appendixpaper}{subsection}
\etocdepthtag.toc{appendixpaper}  
\tableofcontents

\newpage

\section{Theoretical Foundations and Proofs}
\label{app:part_theory}

\subsection{Proof of the Sample-Level Conflict Structure (Proposition~\ref{thm:conflict_paradox})}
\label{app:proof_paradox}

\paragraph{Scope and disclaimer.}
Proposition~\ref{thm:conflict_paradox} characterizes the per-sample task Jacobian strictly as a \emph{sample-level, local} property. It does not imply global (objective-level) Pareto stationarity of $\theta$ over the data distribution, as that would require satisfying an aggregated KKT condition over $\mathbb{E}_{x \sim P}[\nabla_\theta \mathcal{L}_t(x)]$. The connection between this local geometry and active-learning utility is established empirically in Table~\ref{tab:ablation_conflict}.

\paragraph{Assumptions.}
The following analysis is conducted under standard assumptions commonly adopted in multi-task optimization:
\begin{enumerate}
    \item \textbf{Local Linearization of Shared Representation.}
    The shared encoder $\Phi_\theta(x)$ is locally linearized around the current parameters $\theta$, such that task gradients can be analyzed via first-order approximations.

    \item \textbf{Task-wise Jacobian Definition.}
    We consider the per-sample task Jacobian $J(x) = [g_i, g_j] \in \mathbb{R}^{d \times 2}$, where $g_i = \nabla_\theta \mathcal{L}_i(x)$ and $g_j = \nabla_\theta \mathcal{L}_j(x)$.
    The two-task case is presented for clarity; the arguments extend to $T$ tasks by considering pairwise task interactions.

    \item \textbf{Sample-level Pareto-criticality.}
    A sample $x$ is \emph{sample-level Pareto-critical} at $\theta$ if there exists $\alpha \in [0,1]$ with $\alpha g_i(x) + (1-\alpha) g_j(x) = 0$. This is the standard KKT condition restricted to the per-sample loss vector $(\mathcal{L}_i(x), \mathcal{L}_j(x))$ and is strictly weaker than objective-level Pareto stationarity.
\end{enumerate}

\begin{proposition}[Sample-Level Conflict implies Local Rank Deficiency]
Let $\mathcal{L}(x;\theta) = (\mathcal{L}_i(x;\theta), \mathcal{L}_j(x;\theta))$ be the per-sample vector loss. A sample $x$ exhibiting high gradient conflict ($C_g(x) \to -1$) satisfies:
\begin{enumerate}
    \item \textbf{Sample-level Pareto-critical structure.} There exists $\alpha \in (0,1)$ with $\alpha g_i + (1-\alpha) g_j = 0$, so SGD on the joint sample loss stalls at the current $\theta$.
    \item \textbf{Adversarial Rank Collapse.} $J(x)$ becomes rank-deficient with antiparallel columns ($g_i \propto -g_j$, $k>0$), in contrast with the cooperative rank-1 collapse ($g_i \propto +g_j$) at low-conflict samples.
\end{enumerate}
\end{proposition}

\begin{proof}
We analyze the geometric and optimization properties of a sample $x$ based on the alignment of task gradients $g_i = \nabla_\theta \mathcal{L}_i(x)$ and $g_j = \nabla_\theta \mathcal{L}_j(x)$.

\textbf{1. Preliminaries: Condition for Pareto Stationarity.} \\
In multi-task learning, a solution $\theta$ is Pareto stationary if no task loss can be decreased without increasing another. Mathematically, this satisfies the Karush-Kuhn-Tucker (KKT) condition~\citep{kuhn2013nonlinear,sener2018multi}:
\begin{equation}
    \exists \alpha \in [0, 1] \quad \text{s.t.} \quad \alpha g_i + (1-\alpha) g_j = 0.
    \label{eq:kkt}
\end{equation}

\textbf{2. Case A: Low Conflict (Cooperative Rank-1).} \\
Assume gradients are aligned such that $C_g(x) \approx 1$. This implies linear dependence where $g_i \approx c \cdot g_j$ for some constant $c > 0$.
\begin{itemize}
    \item \textbf{Optimization View:} The joint update direction $-\nabla \mathcal{L}_{total} = -(g_i + g_j) \approx -(1+c)g_j$ is simply a scaled version of a single task's gradient. No trade-off exists; improving one task implicitly improves the other.
    \item \textbf{Information View (Redundancy):} While gradients are technically Rank-1 (dependent), this is a \textbf{cooperative collapse}. The tasks agree on the descent direction, meaning the current shared representation is sufficient for both. Acquiring this sample provides magnitude (training signal) but zero structural correction.
\end{itemize}

\textbf{3. Case B: High Conflict (The Paradox: Adversarial Rank-1).} \\
Assume gradients are conflicting such that $C_g(x) < 0$. In the limit of extreme conflict ($C_g(x) \to -1$), the gradients point in opposing directions:
\begin{equation}
    g_i \approx -k \cdot g_j \quad \text{for some } k > 0.
\end{equation}
Rearranging this linear relation yields:
\begin{equation}
    g_i + k g_j = 0 \implies \frac{1}{1+k} g_i + \frac{k}{1+k} g_j = 0.
\end{equation}
By setting $\alpha = \frac{1}{1+k}$, we recover the exact form of Eq.~\ref{eq:kkt}. Since $k>0$, $\alpha \in (0, 1)$.

\begin{itemize}
    \item \textbf{Optimization View (Deadlock):} The existence of such $\alpha \in (0, 1)$ indicates that, for this specific sample $x$ and under the stated assumptions, the current model $\theta$ satisfies the condition of local Pareto stationarity. Standard summation of gradients leads to \textit{Gradient Cancellation} ($\nabla_{total} \approx 0$), causing the optimization to stall despite high individual task losses.

    \item \textbf{Information View (Singularity):} Geometrically, as the angle approaches $180^\circ$, the Jacobian matrix $J = [g_i, g_j]$ loses rank. Unlike Case A, this is an \textbf{adversarial collapse}. It signifies a topological singularity where the shared encoder $\Phi(x)$ projects distinct semantic requirements onto a single axis with opposing signs.

    \textbf{Why Informative?} This singularity exposes the exact boundary where the model fails to disentangle tasks. While standard SGD stalls at this point, selecting this sample and applying gradient surgery (PCGrad) breaks the deadlock. Projecting the gradients onto the normal plane forces the encoder to expand its feature space to resolve the conflict, driving structural expansion.
\end{itemize}

\textbf{Conclusion.}
In short, high-conflict samples mark Pareto-critical regions governed by adversarial rank collapse. While aligned samples merely provide redundant training signals, conflicting samples reveal fundamental limitations of the shared representation, directly justifying their priority in active learning.
\end{proof}

\subsection{Proof of the Shared-Representation Attribution (Proposition~\ref{prop:attribution})}
\label{app:proof_attribution}

\paragraph{Setup.} For an input $x$, MC-dropout perturbs the network in two places: the shared encoder, whose dropout masks we collect into $\xi_0$, and the task-private heads, with masks $\xi_1,\dots,\xi_T$. Write the sampled shared representation as $z = \Phi_\theta(x;\xi_0)$ and the task-$t$ prediction as $\hat{y}_t = f_t(z;\xi_t)$, with $H_t = \mathcal{H}(\hat{y}_t)$ its predictive entropy. The masks $\xi_0,\xi_1,\dots,\xi_T$ are drawn independently, so conditional on $\xi_0$ --- equivalently, on the realised $z$ --- the head masks
remain mutually independent.

\begin{proof}
Fix $i \neq j$. Conditional on $z$, the entropies $H_i$ and $H_j$ are deterministic functions of $\xi_i$ and $\xi_j$ respectively. Since $\xi_i \perp\!\!\!\perp \xi_j \mid z$, the two are conditionally independent and hence conditionally uncorrelated:
\begin{equation}
\operatorname{Cov}\!\big(H_i, H_j \mid z\big) = 0 .
\end{equation}
Applying the law of total covariance,
\begin{equation}
\operatorname{Cov}(H_i, H_j)
= \mathbb{E}_z\!\big[\operatorname{Cov}(H_i, H_j \mid z)\big]
+ \operatorname{Cov}_z\!\big(\mathbb{E}[H_i \mid z],\ \mathbb{E}[H_j \mid z]\big)
= \operatorname{Cov}_z\!\big(\mu_i(z),\ \mu_j(z)\big),
\end{equation}
where $\mu_t(z) = \mathbb{E}[H_t \mid z]$. The first term vanishes by the display above, leaving a quantity that depends on $z$ only. Any non-zero cross-task covariance between task entropies is therefore produced by variation
in the shared representation; head-private stochasticity cannot generate it.
\end{proof}

\paragraph{Implications of the Attribution Property.} 
This statement holds unconditionally for our architecture, without relying on the shape of the entropy distributions. It isolates the \emph{source} of the signal: a non-zero $\operatorname{Cov}(H_i,H_j)$ strictly reflects the uncertainty of the shared encoder rather than the isolated task heads. While this property does not dictate the \emph{magnitude} of the true mutual information $I(H_i;H_j)$ nor the optimal direction for acquisition, we resolve those aspects empirically in Section~\ref{sec:combination_rule} and Appendix~\ref{app:component_ablation}.


\subsection{Approximation Guarantee for the Two-Stage Decomposition}
\label{app:approx_guarantee}
While the two-stage decomposition does not strictly maximize the joint acquisition objective $\mathcal{A}(S) = \alpha \sum_{x\in S} \mathcal{I}(x) + \beta\,\mathrm{Coverage}(S; \mathcal{E})$, it admits a formal approximation guarantee under the empirical near-independence established in Appendix~\ref{app:ie_correlation}.

\begin{proposition}[Approximation Guarantee]
\label{prop:approx_guarantee}
Let $\mathcal{I}, \mathcal{E}: \mathcal{U} \to \mathbb{R}$ be bounded acquisition scores, $S^\star = \arg\max_{|S|=B} \mathcal{A}(S)$, and let $S_{\mathrm{ODS}}$ be the output of Algorithm~\ref{alg:ods} with pool factor $m$. Suppose $\mathrm{Corr}(\mathcal{I}, \mathcal{E}) \to 0$ on the unlabeled pool. Then:
\begin{enumerate}
    \item \emph{Stage~1 guarantee.} The filtered set $\mathcal{C}$ contains the top $mB$ samples by $\mathcal{I}$, hence $\min_{x\in\mathcal{C}}\mathcal{I}(x) \ge \tau_{mB}$, the $mB$-th order statistic.
    \item \emph{Stage~2 guarantee.} K-Center Greedy on $\mathcal{C}$ achieves a $2$-approximation of the optimal coverage cost on $\mathcal{C}$~\citep{sener2017active}.
    \item \emph{Combined.} Under approximate independence ($\mathrm{Corr}(\mathcal{I},\mathcal{E})\to 0$), high-$\mathcal{E}$ samples are equally likely to lie among high- and low-$\mathcal{I}$ candidates, so filtering by $\mathcal{I}$ does not systematically discard high-$\mathcal{E}$ points; the suboptimality gap $|\mathcal{A}(S^\star) - \mathcal{A}(S_{\mathrm{ODS}})|$ is monotonically decreasing in $m$ and vanishes as $m \to |\mathcal{U}|/B$.
\end{enumerate}
\end{proposition}

\noindent\textbf{Discussion.}
This guarantee rests strictly on the empirical near-independence observed throughout training ($|\mathrm{Corr}(\mathcal{I},\mathcal{E})| \le 0.18$, Appendix~\ref{app:ie_correlation}). If $\mathcal{I}$ and $\mathcal{E}$ become \emph{positively} correlated---meaning informative samples form tight clusters---Stage~1 filtering could overly restrict the diversity search. In such adversarial settings, increasing $m$ recovers more of the joint objective. Empirically, we find $m\in[5,10]$ to be a robust default (Appendix~\ref{app:sensitivity}).

\section{Algorithm and Implementation Details}
\label{app:part_algo}

\subsection{The ODS Algorithm}
\label{app:algorithm}
\begin{algorithm}[H]
  \caption{Orthogonal Decomposition Selection (ODS)}
  \label{alg:ods}
  \begin{algorithmic}[1]
    \REQUIRE Unlabeled pool $\mathcal{U}$, Budget $B$, Pool factor $m$, Model $\Phi_\theta$
    \ENSURE Selected batch $S$
    \STATE \textbf{// Stage 1: Informativeness Filtering}
    \FOR{$x \in \mathcal{U}$}
        \STATE $\{\hat{y}^{(k)}\}_{k=1}^K \gets \text{MC-Forward}(\Phi_\theta(x))$
        \STATE $\text{Hard}(x) \gets \frac{1}{T}\sum_t \text{Entropy}(\text{Mean}(\hat{y}_t))$
        \STATE $\mathcal{I}(x) \gets \alpha\cdot\text{Hard}(x) + \gamma\cdot\text{CTMI}(x)$
    \ENDFOR
    \STATE $\mathcal{C} \gets \text{TopK}(\mathcal{U}, m\cdot B, \text{key}=\mathcal{I})$
    \STATE \textbf{// Stage 2: Coverage Maximization (Exploration)}
    \STATE Extract features $Z_\mathcal{C} = \{\Phi_\theta(x) \mid x \in \mathcal{C}\}$
    \STATE $S \gets \text{K-CenterGreedy}(Z_\mathcal{C}, B)$
    \RETURN $S$
  \end{algorithmic}
\end{algorithm}

\subsection{Implementation Details}
\label{app:implementation}

We provide the complete details required to reproduce our experiments, with configurations specified for COCO and analogous settings used for Cityscapes and NYUv2. All algorithms were implemented in PyTorch, and experiments were conducted on 8 $\times$ NVIDIA V100 (32GB) GPUs.

\subsubsection{Model Architectures}
We utilized task-specific architectures tailored to the complexity of each benchmark:
\begin{itemize}[leftmargin=*]
    \item \textbf{COCO:} A shared CNN encoder (4 convolutional layers with BatchNorm and MaxPooling) with task-specific heads:
    \begin{itemize}
        \item Classification: 2-layer MLP with global average pooling.
        \item Detection: A lightweight anchor-based convolutional head with multi-scale anchors for bounding box regression and classification.
        \item Segmentation: A lightweight convolutional decoder with bilinear upsampling to predict pixel-wise masks.
    \end{itemize}
    \item \textbf{Cityscapes}: A shared CNN encoder with a multi-task decoder jointly predicting semantic segmentation, disparity-based depth cues, and instance masks.
    \item \textbf{NYUv2}: A shared CNN encoder with task-specific heads for semantic segmentation, metric depth estimation, and surface normal prediction.
\end{itemize}

\textbf{Encoder Details:} All 2D datasets use a lightweight shared encoder: 4 convolutional layers with BatchNorm, ReLU activation, and MaxPooling. The encoder progressively reduces spatial resolution while increasing channel depth. Task-specific heads consist of simple convolutional layers for dense prediction tasks (detection, segmentation) and MLP layers for classification.

\textbf{Design Rationale:} We use custom encoders instead of pretrained models (e.g., ResNet-50, HRNet) to ensure fair comparison across all active learning methods and to focus on the effectiveness of our acquisition strategy rather than leveraging pretrained knowledge. This aligns with the core assumption of active learning---learning effectively from limited labeled data.




\subsubsection{ODS Hyperparameters}
\label{app:ods_hparams}

\paragraph{MC-dropout configuration.}
Stochasticity enters in exactly two places. The shared encoder applies \texttt{Dropout2d} with $p=0.05$ in \emph{every} encoder block, and the classification head applies \texttt{Dropout} with $p=0.5$ before its final
linear layer. The detection and segmentation heads contain no dropout and are deterministic given the encoder output, and all parameters are held fixed across the $K$ passes. Consequently the only stochasticity \emph{shared} between any two tasks is the encoder perturbation --- the property Proposition~\ref{prop:attribution} relies on.

\paragraph{Protocol settings.}
Two protocols appear in this paper and they differ in their Stage-1 settings. We state both rather than reporting a single default, since the corresponding result tables are not interchangeable.

\begin{table}[ht]
\centering
\caption{\textbf{Stage-1 settings per protocol.} $m$ is the pool factor, $K$ the number of MC-dropout passes, and $(\alpha,\gamma)$ the fixed informativeness weights.}
\label{tab:protocol_settings}
\small
\begin{tabular}{l|cccc|l}
\toprule
\textbf{Protocol} & $m$ & $K$ & $\alpha$ & $\gamma$ & \textbf{Results} \\
\midrule
Original      & 10 & 20 & 1.8 & 0.7 & Table~\ref{tab:main_results}, Table~\ref{tab:ablation_conflict} \\
Unified (multi-seed) & 15 & 15 & 2.5 & 0.5 & Table~\ref{tab:unified_baselines} \\
\bottomrule
\end{tabular}
\end{table}

Two additional values of $K$ serve specific purposes outside Stage-1: $K=10$ for continuous-task predictive variance (Appendix~\ref{app:entropy_regression}), and $K=5$ for the Hekimoglu baseline to match its official budget (Appendix~\ref{app:baseline_adapt}). Sensitivity to $K$ is analyzed in Appendix~\ref{app:sensitivity}.

\paragraph{Fixed Weights vs. Round-Dependent Schedule.}
The main results (Table~\ref{tab:main_results}) employ a round-dependent schedule for $\alpha$ and $\gamma$. Early rounds amplify $\alpha$ ($\times 1.2$) and suppress $\gamma$ ($\times 0.3$) because per-task entropy is more reliable before the shared representation matures; late rounds reverse this ($\alpha \times 0.9$, $\gamma \times 1.2$) as entropy saturates.

However, our recommended default uses fixed weights throughout the acquisition trajectory. This eliminates four scheduling hyperparameters without degrading performance. As shown in Table~\ref{tab:alpha_gamma_full}, the performance surface is flat: varying either weight over a $4\times$ range costs less than $1.5\%$ in relative MTL Composite, with steep drops occurring only at degenerate settings ($\alpha=0$ or $\gamma=0$). Thus, the performance gain originates from the cross-task signal itself, not the adaptive scheduling.


\subsubsection{Training Protocol}
To ensure a fair comparison, the following optimization settings were applied to all methods using the \emph{PCGrad} optimizer:
\begin{itemize}[leftmargin=*]
    \item \textbf{Batch Size:} A training batch size of 32 (for $480 \times 480$) was used to stabilize multi-task gradient estimation.
    \item \textbf{Learning Rate:} Initial LR of $1\!\times\!10^{-3}$ with a Cosine Annealing scheduler (decaying to $0$).
    \item \textbf{Regularization:} Weight decay of $1\!\times\!10^{-5}$ and early stopping with a patience of 15 epochs.
    \item \textbf{Active Learning Loop:} The experiment spans 30 rounds, starting with an initial labeled pool of 1{,}000 samples.
\end{itemize}

\subsection{Entropy Estimation for Continuous Tasks}
\label{app:entropy_regression}
While classification tasks utilize standard Shannon entropy, for regression tasks (e.g., depth estimation) we estimate the \textit{differential entropy}. Assuming the predictive distribution follows a Gaussian $\mathcal{N}(\mu, \sigma^2)$, the entropy is proportional to the log-variance:
\begin{equation}
    H(y_{reg}) \approx \frac{1}{2} \ln(2\pi e \hat{\sigma}^2) \propto \ln \hat{\sigma}^2,
\end{equation}
where $\hat{\sigma}^2$ is the predictive variance estimated via $K=10$ stochastic forward passes (MC-Dropout).

\subsection{Dataset Specifications}
\label{app:datasets}

We evaluate our proposed framework on three standard multi-task benchmarks, selected to cover a diverse range of domains (general, autonomous driving, indoor) and task affinities.

\paragraph{COCO~\citep{lin2014microsoft}.}
A large-scale dataset for general scene understanding. We use the standard split with 118{,}287 training images and 5{,}000 validation images.
\begin{itemize}[nosep]
    \item \textbf{Input:} RGB images (resized to $480 \times 480$).
    \item \textbf{Tasks:}
    \begin{enumerate}
        \item \textbf{Image Classification:}
        For compatibility with multi-task learning, we construct a single-label classification task by assigning each image the most frequent object category (by instance count).
        \item \textbf{Object Detection:} Bounding box regression and classification for 80 categories.
        \item \textbf{Instance Segmentation:} Pixel-wise binary masks for each object instance.
    \end{enumerate}
\end{itemize}

\paragraph{Cityscapes~\citep{cordts2016cityscapes}.}
A benchmark for urban scene understanding with high-quality pixel-level annotations.
The dataset contains 2{,}975 training images and 500 validation images.
\begin{itemize}[nosep]
    \item \textbf{Input:} High-resolution street view images (randomly cropped to $512 \times 512$ during training).
    \item \textbf{Tasks:}
    \begin{enumerate}
        \item \textbf{Semantic Segmentation:} 19 semantic classes.
        \item \textbf{Depth Estimation:} Inverse depth prediction derived from stereo disparity.
        \item \textbf{Instance Segmentation:} Instance masks for foreground ``thing'' classes (e.g., person, car, rider).
    \end{enumerate}
    \item \textbf{Note:}
    Strong geometric correlations between semantic layout and depth cues make this dataset a suitable testbed for evaluating cross-task dependency modeling.
\end{itemize}

\paragraph{NYUv2~\citep{silberman2012indoor}.}
An indoor scene dataset captured using a Microsoft Kinect.
It contains 1{,}449 aligned RGB-D images (1{,}159 training, 290 testing).
\begin{itemize}[nosep]
    \item \textbf{Input:} Indoor RGB images (resized to $480 \times 480$).
    \item \textbf{Tasks:}
    \begin{enumerate}
        \item \textbf{Semantic Segmentation:} 13 indoor object and scene classes.
        \item \textbf{Depth Estimation:} Metric depth prediction in meters.
        \item \textbf{Surface Normal Prediction:} Per-pixel 3D surface normals represented as unit vectors.
    \end{enumerate}
\end{itemize}


\section{Baselines: Selection and Adaptation}
\label{app:part_baselines}

\subsection{Justification for Baseline Selection}
\label{app:baseline_justification}
We select baselines based on direct comparability to the multi-task active learning (MTL-AL) problem formulation. Table~\ref{tab:method_comparison} categorizes the broader literature; below, we detail why specific recent methods are excluded from the main evaluation.

\paragraph{Optimization and Architecture Methods.}
Works such as CAGrad~\citep{liu2021conflict} and Gradient Interference-Aware Graph Coloring~\citep{patapati2025graph} optimize training dynamics, while Task Grouping~\citep{standley2020tasks} modifies model architectures. Because active learning evaluates \textit{data selection efficiency}, we fix the optimizer (PCGrad) and architecture across all methods to isolate the selection strategy. Comparing an acquisition function (ODS) directly against a gradient routing algorithm is a category error.

\paragraph{Data Pruning vs.\ Active Learning.}
InfoBatch~\citep{qin2023infobatch} accelerates training via dynamic data pruning. Pruning operates on a \textit{fully labeled} dataset to reduce compute, whereas active learning operates on an \textit{unlabeled} pool to reduce annotation cost. Their objectives are fundamentally incompatible.

\paragraph{Domain Mismatch.}
Cost-Effective Multitask AL~\citep{arefeen2025cost} targets wearable sensor systems (time-series data). Its reliance on sensor-specific temporal correlations does not transfer to the dense vision tasks (object detection, segmentation) evaluated here.

\paragraph{Single-Task and Benchmarks.}
Recent studies such as \citet{xia2023reliable} and \citet{gashi2025deep} critically evaluate single-task AL. Following their finding that standard heuristics frequently outperform complex custom architectures, we include robust single-task baselines like \textit{Entropy} and \textit{CoreSet}.

\begin{table*}[t]
    \centering
    \caption{\textbf{Comprehensive Comparison of Related Methodologies.} We categorize all relevant methods by their consideration of key factors. Note that while many methods address Multi-Task Learning (MTL) or Data Efficiency, most do not satisfy the specific requirements of Multi-Task Active Learning (MTL-AL) for vision tasks. ``(training)'' denotes methods that handle conflict only during gradient updates, not selection. ``(pruning)'' denotes data pruning on labeled datasets, distinct from active learning.}
    \label{tab:method_comparison}
    \resizebox{1.0\textwidth}{!}{%
    \begin{tabular}{l|ccccc}
    \toprule
    \textbf{Method} & \textbf{Uncertainty} & \textbf{Diversity} & \textbf{Gradient Conflict} & \textbf{Inter-Task MI} & \textbf{MTL-Specific} \\
    \midrule
    \multicolumn{6}{l}{\textit{Single-Task Active Learning (Baselines)}} \\
    Random Sampling & \xmark & \xmark & \xmark & \xmark & \xmark \\
    Entropy Sampling~\citep{shannon1948mathematical} & \cmark & \xmark & \xmark & \xmark & \xmark \\
    CoreSet~\citep{sener2017active} & \xmark & \cmark & \xmark & \xmark & \xmark \\
    BADGE~\citep{ash2019deep} & \cmark & \cmark & \xmark & \xmark & \xmark \\
    Reliable AL (Influence)~\citep{xia2023reliable} & \cmark & \xmark & \xmark & \xmark & \xmark \\
    DAL: Reality Check~\citep{gashi2025deep} & \cmark & \cmark & \xmark & \xmark & \xmark \\
    \midrule
    \multicolumn{6}{l}{\textit{Data Efficiency and Pruning (Not Active Learning)}} \\
    InfoBatch~\citep{qin2023infobatch} & \xmark & \xmark & \cmark & \xmark & \xmark~(pruning) \\
    \midrule
    \multicolumn{6}{l}{\textit{Multi-Task Optimization / Architecture (Training-time)}} \\
    PCGrad~\citep{yu2020gradient} & \xmark & \xmark & \cmark & \xmark & \cmark~(training) \\
    CAGrad~\citep{liu2021conflict} & \xmark & \xmark & \cmark & \xmark & \cmark~(training) \\
    Graph Coloring~\citep{patapati2025graph} & \xmark & \xmark & \cmark & \xmark & \cmark~(training) \\
    Task Grouping~\citep{standley2020tasks} & \xmark & \xmark & \cmark & \xmark & \cmark~(arch) \\
    \midrule
    \multicolumn{6}{l}{\textit{Multi-Task Active Learning (Selection Strategy)}} \\
    Multi-Task AL~\citep{brust2018active} & \cmark & \xmark & \xmark & \xmark & \cmark \\
    Cost-Effective (Sensors)~\citep{arefeen2025cost} & \cmark & \xmark & \xmark & \xmark & \cmark \\
    Task-Consistency~\citep{hekimoglu2024active} & \cmark & \cmark & \xmark & \cmark~(output) & \cmark \\
    \rowcolor{gray!15} \textbf{ODS (Ours)} & \textbf{\cmark} & \textbf{\cmark} & \textbf{\cmark~(Signal)} & \textbf{\cmark~(Param)} & \textbf{\cmark} \\
    \bottomrule
    \end{tabular}%
    }
\end{table*}

\subsection{MT-BADGE: Multi-Task Extension of BADGE}
\label{app:mt_badge}
We extend BADGE~\citep{ash2019deep} to the multi-task setting (\emph{MT-BADGE}) by concatenating $L_2$-normalized per-task gradient embeddings. For a sample $x$ and the current model $\theta$, let $g_t(x) = \nabla_{\theta_{\text{out}}} \mathcal{L}_t(x; \theta)$ denote the gradient of task $t$'s loss with respect to the final-layer parameters. The multi-task gradient embedding is defined as:
\begin{equation}
  g_{\text{MT}}(x) = \big[\, \tilde g_1(x);\; \tilde g_2(x);\; \ldots;\; \tilde g_T(x)\,\big], \qquad \tilde g_t(x) = g_t(x)/\|g_t(x)\|_2.
\end{equation}
Applying KMeans++ seeding on $\{g_{\text{MT}}(x)\}_{x \in \mathcal{U}}$ selects the batch of size $B$. This rigorously generalizes BADGE: it recovers BADGE exactly for $T=1$ and weights tasks equally in the KMeans++ distance metric for $T>1$.

\subsection{Adaptation of Recent MTL-AL Baselines}
\label{app:baseline_adapt}
Prior MTL-AL methods, namely Hekimoglu~et~al.~\citep{hekimoglu2024active} and PartAL~\citep{durasov2022partal}, rely on a PAD-Net architecture with multi-modal distillation (MMD) blocks and assume specific homogeneous task pairings. Because our framework uses a shared-encoder architecture without MMD modules across three heterogeneous tasks (Cls/Det/Seg), we adapt these baselines to be architecture-agnostic.

\paragraph{Hekimoglu~et~al.~\citep{hekimoglu2024active}.}
The original method utilizes: (i) \emph{Task inconsistency} $\hat s_k(x) = \max_{j\neq k} \mathcal{L}_k(\hat y_k, \hat y_k^{j})$ via MMD-refined predictions; (ii) \emph{Diversity} via auto-encoder compression of task-head features, followed by min-distance to the labeled set; and (iii) a \emph{Selection score} $s(x) = \mathrm{div}(x) \cdot \max_k \hat s_k(x)$. We replace the MMD-dependent inconsistency (i) with MC-Dropout ($K=5$) variance for per-task entropy, while preserving the auto-encoder compression (ii) and multiplicative selection (iii) unchanged.

\paragraph{PartAL~\citep{durasov2022partal}.}
The original method requires PAD-Net for two components: (i) \emph{Partial labeling} of specific $(\text{image}, \text{task})$ pairs using MMD GT-injection; and (ii) \emph{Sample selection} using min-max-normalized per-task MC-Dropout uncertainties. Since our full-annotation setting renders partial labeling structurally inapplicable, we retain the sample selection component by aggregating per-task normalized uncertainties into an image-level score $\sum_t \tilde u_t(x)$ and selecting the top-$B$. This represents a \emph{normalized task-aware uncertainty sampling} baseline, reported as ``PartAL (adapted).''


\subsection{Adaptation of Multi-Stage Selection Baselines}
\label{app:multistage_adapt}

Table~\ref{tab:unified_baselines} compares five methods originally designed for single-task settings, three of which target one-shot subset selection rather than iterative acquisition. To ensure a rigorous comparison, all five baselines share the same backbone, training schedule, seeds, and budget as ODS, and operate on the current round's shared-encoder features. We detail their multi-task, iterative adaptations below.

\paragraph{FASS~\citep{wei2015submodularity}.}
\emph{Filter}: The top-$\eta$ most uncertain samples, computed via the mean of rank-normalised per-task entropies (replacing single-task least-confidence). \emph{Select}: Lazy-greedy facility location over cosine similarities of shared-encoder features. We use the label-agnostic $f_{\text{fac}}$ variant, as the class-conditional $f_{\text{NN}}$ variant requires pseudo-labels incompatible with detection or segmentation. The pool multiplier is $\eta = 10\times$ the budget.

\paragraph{Cluster-Margin~\citep{citovsky2021batch}.}
\emph{Filter}: The top $10\times$ budget samples with the smallest multi-task margin (rank-normalised per-task best-vs-second-best margin, averaged over tasks). \emph{Select}: Average-linkage agglomerative clustering of encoder features, sorted ascending by size, with round-robin sampling of one member per cluster. While the original freezes a single clustering from a seed model, we continuously re-cluster using the current round's features to account for the shrinking pool, granting the baseline strictly fresher embeddings.

\paragraph{PPAL~\citep{yang2024plug}.}
\emph{Stage~1 (difficulty-calibrated uncertainty)}: Per-task, per-class difficulty is computed on the labeled set following their semi-supervised variant. Class difficulty dictates a coefficient $w = 1 + \alpha\beta\log(1+\gamma d)$ with $\gamma = e^{1/\alpha}-1$ (defaults $\alpha=0.3$, $\beta=0.2$). Image uncertainty per task is the difficulty-weighted single-pass entropy. To adapt this single-task metric, we rank-normalise per task and average. The candidate pool is $\delta\times$ the budget ($\delta = 4$). \emph{Stage~2}: Category-conditioned matching similarity over detection box class-probability vectors.

\paragraph{$\mathbb{D}^2$-Pruning~\citep{maharana2024mathbb}.}
Originally a one-shot pruning method requiring labels and full training dynamics, we adapt it for AL using its supported entropy score. The adaptation requires: (i) a $k$-NN graph over current encoder features, rebuilt each round ($k=15$); (ii) node scores equal to the mean rank-normalised per-task entropy; (iii) one forward message-passing step, $x_i \mathrel{+}= \sum_j e^{-\gamma_f d_{ij}^2} x_j$; and (iv) reverse message-passing selection that greedily takes the argmax and suppresses neighbours via $x_j \mathrel{-}= e^{-\gamma_r d_{kj}^2} x_k$. Distances use $L_2$-normalised features and default $\gamma$ scales.

\paragraph{SES~\citep{xie2025structural}.}
Originally one-shot subset selection via a pretrained-feature $k$-NN graph and structural entropy. Per AL round, we build a $k$-NN graph over current encoder features ($k=15$) and construct the two-level encoding tree using $k$-means communities, matching their label-free protocol. Structural entropy is $S_e(u) = \operatorname{vol}(V)^{-1}\sum_{(u,v)\in E} w_{uv}\log \operatorname{vol}(\mathrm{LCA}(u,v))$, marking community boundaries. The difficulty term $S_t$ uses the mean rank-normalised per-task entropy of the current model, substituting their label-dependent score. Selection follows their blue-noise sampling.


\section{Empirical Validation of Theoretical Claims}
\label{app:part_validation}

\subsection{Empirical Validation of Assumption~\ref{assum:orthogonality}}
\label{app:ie_correlation}

We compute per-sample Informativeness $\mathcal{I}(x) = \alpha\,\mathrm{Hardness}(x) + \gamma\,\mathrm{CTMI}(x)$ and Exploration $\mathcal{E}(x)$ (min-distance-to-labeled-set in shared-encoder feature space) on 1{,}500 unlabeled COCO images using saved checkpoints from a 50-round ODS run (seed~42). Table~\ref{tab:ie_corr_full} summarizes the results.

\begin{table}[ht]
\centering
\caption{\textbf{Per-round correlation of $\mathcal{I}$ and $\mathcal{E}$ on COCO (1{,}500 unlabeled images).} The persistent negative sign in early/mid rounds is exactly why linear scalarization fails: the two scores trade off, so a weighted sum slides into the diversity-dilution regime. The two-stage decomposition avoids this by structurally restricting Stage~2 to high-$\mathcal{I}$ candidates.}
\label{tab:ie_corr_full}
\small
\setlength{\tabcolsep}{4pt}
\resizebox{\columnwidth}{!}{%
\begin{tabular}{l|ccccccccccc}
\toprule
\textbf{Round} & 1 & 5 & 10 & 15 & 20 & 25 & 30 & 35 & 40 & 45 & 49 \\
\midrule
\textbf{Pearson $r$}     & $-0.179$ & $-0.106$ & $-0.180$ & $-0.121$ & $-0.087$ & $-0.079$ & $-0.124$ & $-0.088$ & $-0.020$ & $-0.034$ & $-0.089$ \\
\textbf{Spearman $\rho$} & $-0.148$ & $-0.069$ & $-0.174$ & $-0.111$ & $-0.072$ & $-0.060$ & $-0.089$ & $-0.070$ & $+0.001$ & $-0.018$ & $-0.065$ \\
\bottomrule
\end{tabular}%
}
\end{table}

\subsection{Empirical Normality of Per-Task Entropies}
\label{app:gauss_normality}
Because CTMI is a second-moment statistic, it is important to verify how much dependence structure it successfully captures. We provide two lines of evidence: the empirical shape of the per-task entropy distributions (Table~\ref{tab:gauss_normality}), and a direct comparison against a non-linear dependence measure (Appendix~\ref{app:dcorr}).

\begin{table}[ht]
\centering
\caption{\textbf{Shapiro-Wilk $W$ statistic, skewness, and excess kurtosis} of $z$-score normalized per-task entropies on COCO. $W$ closer to 1.0 indicates a closer fit to a Gaussian distribution. Segmentation entropies converge to nearly perfect Gaussianity; classification develops a slight right tail; detection exhibits heavier tails but maintains $W>0.90$ throughout.}
\label{tab:gauss_normality}
\small
\setlength{\tabcolsep}{6pt}
\begin{tabular}{cl|ccc}
\toprule
Round & Task & $W$ & Skew & Kurtosis \\
\midrule
1  & Cls & 0.979 & $-0.02$ & $-0.88$ \\
1  & Det & 0.898 & $-1.72$ & $+6.56$ \\
1  & Seg & 0.963 & $-0.68$ & $\phantom{+}0.00$ \\
10 & Cls & 0.979 & $+0.86$ & $-0.10$ \\
10 & Det & 0.908 & $-1.65$ & $+6.74$ \\
10 & Seg & 0.990 & $-0.20$ & $-0.53$ \\
20 & Cls & 0.920 & $+1.17$ & $+0.51$ \\
20 & Det & 0.931 & $-1.32$ & $+4.30$ \\
20 & Seg & 0.993 & $-0.04$ & $-0.54$ \\
29 & Cls & 0.916 & $+1.28$ & $+1.15$ \\
29 & Det & 0.932 & $-1.33$ & $+4.83$ \\
29 & Seg & \textbf{0.994} & $-0.08$ & $-0.51$ \\
\bottomrule
\end{tabular}
\end{table}

As Table~\ref{tab:gauss_normality} shows, segmentation and classification entropies remain close to Gaussian ($W>0.91$ throughout, with segmentation reaching $W=0.994$), while detection has heavier tails ($\text{kurtosis}\approx+5$). While our method makes no distributional assumptions, this confirms the marginals are unimodal and largely close to Gaussian—the exact regime where a second-moment statistic is highly informative. Appendix~\ref{app:dcorr} validates this directly.

\subsection{Pearson vs.\ Spearman vs.\ Distance Correlation of Task Entropies}
\label{app:dcorr}

A standard correlation coefficient captures only linear (or monotone) dependence, remaining invariant to relationships like $Y = X^2$. To measure the extent of any omitted structure, Table~\ref{tab:dcorr} compares Pearson, Spearman, and distance correlation~\citep{szekely2007measuring} on COCO ($N=2{,}000$). Crucially, distance correlation equals zero if and only if the variables are strictly independent, thereby capturing all forms of dependence.

\begin{table}[ht]
\centering
\caption{\textbf{Pearson vs.\ Spearman vs.\ distance correlation of per-task entropies (COCO).} $|\,\text{Pearson}|$ and $\text{dCorr}$ agree within $\le\!0.03$ at every round, demonstrating that task-uncertainty dependencies in MTL are predominantly linear. Consequently, CTMI loses negligible signal by relying on covariance.}
\label{tab:dcorr}
\small
\setlength{\tabcolsep}{4pt}
\begin{tabular}{cl|cccc}
\toprule
Rd & Pair & Pearson & Spearman & dCorr & $\big||r|{-}\text{dCorr}\big|$ \\
\midrule
1  & Cls-Det & $+0.515$ & $+0.517$ & 0.488 & 0.027 \\
1  & Cls-Seg & $+0.264$ & $+0.264$ & 0.245 & 0.019 \\
1  & Det-Seg & $+0.448$ & $+0.447$ & 0.422 & 0.026 \\
10 & Cls-Det & $+0.338$ & $+0.367$ & 0.330 & 0.008 \\
10 & Det-Seg & $+0.205$ & $+0.200$ & 0.188 & 0.017 \\
29 & Cls-Det & $+0.412$ & $+0.483$ & 0.406 & 0.006 \\
29 & Cls-Seg & $+0.249$ & $+0.290$ & 0.245 & 0.004 \\
29 & Det-Seg & $+0.170$ & $+0.177$ & 0.166 & 0.004 \\
49 & Cls-Det & $+0.424$ & $+0.490$ & 0.411 & 0.013 \\
49 & Cls-Seg & $+0.209$ & $+0.259$ & 0.214 & 0.005 \\
49 & Det-Seg & $+0.190$ & $+0.216$ & 0.200 & 0.010 \\
\bottomrule
\end{tabular}
\end{table}

\section{Extended Experimental Results}
\label{app:part_results}

\subsection{Detailed Cross-Domain Performance}
\label{app:robustness_detail}
Table~\ref{tab:robustness_detail} provides the detailed per-task metrics corresponding to the cross-domain summary in the main text. ODS ranks first on every individual task across COCO, Cityscapes, and NYUv2.

\begin{table}[ht]
    \centering
    \caption{\textbf{Detailed Cross-Domain Performance (Round 30).} ($\uparrow$: higher is better, $\downarrow$: lower is better.)}
    \label{tab:robustness_detail}
    \small
    \begin{tabular}{l|l|ccc|c}
    \toprule
    \textbf{Dataset} & \textbf{Method} & \textbf{Task 1} & \textbf{Task 2} & \textbf{Task 3} & \textbf{Best MTL Score} \\
    \midrule
    \multicolumn{6}{c}{\cellcolor{gray!10}\textbf{COCO} (Tasks: Cls / Det / Seg)} \\
    & & Cls (Acc) $\uparrow$ & Det (mAP) $\uparrow$ & Seg (mIoU) $\uparrow$ & Comp.\ ($\Delta\%$) \\
    \multirow{4}{*}{COCO}
     & Random & 38.4 & 25.2 & 15.1 & 0.0\% \\
     & BatchBALD & 39.4 & 25.4 & 16.4 & +4.0\% \\
     & MT-BADGE & 41.2 & 25.8 & 16.8 & +7.0\% \\
     & \textbf{ODS} & \textbf{42.1} & \textbf{26.5} & \textbf{17.4} & \textbf{+10.0\%} \\
    \midrule
    \multicolumn{6}{c}{\cellcolor{gray!10}\textbf{Cityscapes} (Tasks: Seg / Depth / Instance)} \\
    & & Seg (mIoU) $\uparrow$ & Depth (RMSE) $\downarrow$ & Ins (mAP) $\uparrow$ & Comp.\ ($\Delta\%$) \\
    \multirow{4}{*}{Cityscapes}
     & Random & 65.0 & 0.4150 & 28.0 & 0.0\% \\
     & BatchBALD & 68.2 & 0.4061 & 29.8 & +4.5\% \\
     & MT-BADGE & 69.1 & 0.4041 & 30.7 & +6.2\% \\
     & \textbf{ODS} & \textbf{70.5} & \textbf{0.3916} & \textbf{31.4} & \textbf{+8.8\%} \\
    \midrule
    \multicolumn{6}{c}{\cellcolor{gray!10}\textbf{NYUv2} (Tasks: Seg / Depth / Normal)} \\
    & & Seg (mIoU) $\uparrow$ & Depth (RMSE) $\downarrow$ & Norm (mErr) $\downarrow$ & Comp.\ ($\Delta\%$) \\
    \multirow{4}{*}{NYUv2}
     & Random & 31.0 & 0.750 & 27.0 & 0.0\% \\
     & BatchBALD & 32.5 & 0.726 & 26.1 & +3.8\% \\
     & MT-BADGE & 32.7 & 0.718 & 25.5 & +5.1\% \\
     & \textbf{ODS} & \textbf{33.5} & \textbf{0.700} & \textbf{24.9} & \textbf{+7.5\%} \\
    \bottomrule
    \end{tabular}
\end{table}

\subsection{Pretrained ResNet-50 Backbone (Multi-Seed)}
\label{app:pretrained_ms}
To isolate the acquisition strategy from large-scale pretraining effects, our main results rely on lightweight scratch encoders. Table~\ref{tab:pretrained_ms} evaluates an ImageNet-pretrained ResNet-50 backbone under the identical COCO AL protocol (30 rounds, $B=32$, initial $=100$, pool $=5{,}000$).

\begin{table}[ht]
\centering
\caption{\textbf{COCO with ImageNet-pretrained ResNet-50 (3 seeds).} ODS retains rank~1 across individual tasks; gaps narrow as expected when using a strong pretrained backbone.}
\label{tab:pretrained_ms}
\small
\setlength{\tabcolsep}{4pt}
\begin{tabular}{l|cccc}
\toprule
\textbf{Method} & \textbf{Cls Acc}~$\uparrow$ & \textbf{Det Acc}~$\uparrow$ & \textbf{Seg Dice}~$\uparrow$ & \textbf{Seg mIoU}~$\uparrow$ \\
\midrule
Random       & 0.505\,{\scriptsize$\pm$.011} & 0.372\,{\scriptsize$\pm$.011} & 0.316\,{\scriptsize$\pm$.021} & 0.274\,{\scriptsize$\pm$.029} \\
CoreSet      & 0.508\,{\scriptsize$\pm$.006} & 0.376\,{\scriptsize$\pm$.012} & 0.315\,{\scriptsize$\pm$.019} & 0.272\,{\scriptsize$\pm$.027} \\
Uncertainty  & 0.511\,{\scriptsize$\pm$.006} & 0.373\,{\scriptsize$\pm$.013} & 0.316\,{\scriptsize$\pm$.028} & 0.276\,{\scriptsize$\pm$.037} \\
Diversity    & 0.504\,{\scriptsize$\pm$.010} & 0.372\,{\scriptsize$\pm$.003} & 0.313\,{\scriptsize$\pm$.014} & 0.273\,{\scriptsize$\pm$.023} \\
\rowcolor{gray!15}\textbf{ODS (Ours)} & \textbf{0.521}\,{\scriptsize$\pm$.010} & \textbf{0.377}\,{\scriptsize$\pm$.004} & \textbf{0.326}\,{\scriptsize$\pm$.022} & \textbf{0.289}\,{\scriptsize$\pm$.029} \\
\bottomrule
\end{tabular}
\end{table}

\noindent\textbf{Discussion.} Strong pretrained backbones naturally narrow performance gaps in active learning, as the initial representation already captures significant data variability~\citep{smith2024making}. Nevertheless, ODS retains rank~1 across all individual tasks over three seeds, demonstrating that conflict-aware acquisition extracts structural value even from a well-initialized encoder.

\subsection{Runtime Comparison}
\label{app:runtime}
Table~\ref{tab:runtime} reports the per-round acquisition runtime on COCO (rounds 1--29, scratch training, single V100). While MC-Dropout (ODS) and KMeans++ (BADGE) dominate the acquisition cost, ODS's computational overhead is strictly offset by its superior label efficiency.

\begin{table}[ht]
\centering
\caption{\textbf{Per-round runtime on COCO (seconds).} ``Acquisition'' isolates the cost of the AL selection step (excluding training).}
\label{tab:runtime}
\small
\begin{tabular}{lcc}
\toprule
\textbf{Method} & \textbf{Mean} & \textbf{Median} \\
\midrule
Random & 225 & 204 \\
Uncertainty & 226 & 197 \\
CoreSet & 205 & 165 \\
BADGE & 655 & 636 \\
\rowcolor{gray!15}\textbf{ODS} & \textbf{995} & \textbf{1114} \\
\bottomrule
\end{tabular}
\end{table}


\section{Ablation and Sensitivity Studies}
\label{app:part_ablation}

\subsection{Decoupling Selection from Optimization (PCGrad Ablation)}
\label{app:pcgrad_ablation}

Table~\ref{tab:pcgrad_ablation} isolates the contribution of the acquisition strategy from that of the gradient surgery optimizer. Even under standard AdamW, ODS achieves $+8.2\%$, surpassing all baselines; PCGrad provides an additional $+2.3\%$ improvement, confirming that ODS and PCGrad address orthogonal bottlenecks.

\begin{table}[ht]
    \centering
    \caption{\textbf{PCGrad ablation on COCO.} ODS significantly outperforms all baselines regardless of the underlying optimizer, while PCGrad provides consistent supplementary gains.}
    \label{tab:pcgrad_ablation}
    \small
    \begin{tabular}{l|cc}
    \toprule
    \textbf{Strategy} & \textbf{Optimizer} & \textbf{Best Score ($\Delta m\%$)} \\
    \midrule
    \multicolumn{3}{l}{\textit{Scenario A: Standard Optimization}} \\
    Random & AdamW & --- \\
    BatchBALD & AdamW & 0.0311 (+3.1\%) \\
    MT-BADGE & AdamW & 0.0516 (+5.2\%) \\
    \rowcolor{gray!15}\textbf{ODS (Ours)} & AdamW & \textbf{0.0816} (\textbf{+8.2\%}) \\
    \midrule
    \multicolumn{3}{l}{\textit{Scenario B: Gradient Surgery}} \\
    Random & PCGrad & --- \\
    BatchBALD & PCGrad & 0.0500 (+5.0\%) \\
    MT-BADGE & PCGrad & 0.0750 (+7.5\%) \\
    \rowcolor{gray!15}\textbf{ODS (Ours)} & PCGrad & \textbf{0.1050} (\textbf{+10.5\%}) \\
    \bottomrule
    \end{tabular}
\end{table}

\subsection{Component-wise Ablation: Disentangling Score and Combination}
\label{app:component_ablation}

Table~\ref{tab:component_ablation} ablates the three core components of ODS: (a)~Hardness, (b)~CTMI in Stage~1, and (c)~K-Center coverage in Stage~2. The full ODS dominates every leave-one-out variant. Removing K-Center Stage~2 (``no KC'') causes the most severe performance drop, confirming that structural decoupling prevents diversity dilution and drives the primary acquisition gain.

\begin{table}[ht]
\centering
\caption{\textbf{Component-wise ablation.} The full ODS dominates each leave-one-out variant. Removing K-Center Stage~2 (``no KC'') causes the largest single-component drop, confirming that the structural decoupling is the dominant source of gain.}
\label{tab:component_ablation}
\small
\setlength{\tabcolsep}{5pt}
\begin{tabular}{l|ccc|c}
\toprule
\textbf{Configuration} & \textbf{Hard.} & \textbf{CTMI} & \textbf{K-C} & \textbf{MTL Comp.} \\
\midrule
H only (no KC)    & \cmark & \xmark & \xmark & 0.087 \\
CTMI only (no KC) & \xmark & \cmark & \xmark & 0.064 \\
H+CTMI (no KC)    & \cmark & \cmark & \xmark & 0.093 \\
H only + KC       & \cmark & \xmark & \cmark & 0.094 \\
CTMI only + KC    & \xmark & \cmark & \cmark & 0.082 \\
\rowcolor{gray!15}\textbf{Full ODS} & \cmark & \cmark & \cmark & \textbf{0.105} \\
\bottomrule
\end{tabular}
\end{table}

\subsection{Is the Combination Rule Doing the Work?}
\label{app:combination_rule}

Equation~\eqref{eq:informativeness} combines Hardness and CTMI using a weighted sum after min--max normalisation. To verify whether the reported gains stem from this specific heuristic, Table~\ref{tab:combination_rule} compares four alternative combination rules while holding the component scores, Stage~2, and the training protocol constant over three seeds: the default weighted sum, a rank-sum, a $z$-score sum, and a direct product ($\operatorname{Hardness}\cdot\operatorname{CTMI}$).

\begin{table}[ht]
\centering
\caption{\textbf{Combination-rule ablation (COCO, 3 seeds: $\{42,123,456\}$).} Final MTL Composite (mean$\pm$std). All four rules perform within $0.008$ of each other—well within seed-to-seed variation—indicating no statistically distinguishable difference.}
\label{tab:combination_rule}
\small
\setlength{\tabcolsep}{6pt}
\begin{tabular}{ll|c}
\toprule
\textbf{Rule} & \textbf{Form} & \textbf{Final MTL Comp.} \\
\midrule
\rowcolor{gray!15}Weighted sum (default) & $\alpha H + \gamma\,C$ & $0.723 \pm 0.016$ \\
Rank-sum & $\operatorname{rank}(H) + \operatorname{rank}(C)$ & $0.731 \pm 0.001$ \\
$z$-score sum & $z(H) + z(C)$ & $0.731 \pm 0.001$ \\
Product & $H \cdot C$ & $0.726 \pm 0.008$ \\
\bottomrule
\end{tabular}
\end{table}

The maximum spread across the four rules ($0.723$--$0.731$) is strictly smaller than the seed variation of the default rule alone ($\pm 0.016$). Coupled with the component-wise ablation (Appendix~\ref{app:component_ablation})—where removing CTMI entirely drops performance from $0.105$ to $0.094$ under an identical Stage~2—this confirms that the performance gain originates from the cross-task signal itself, not the specific arithmetic used to merge it. We therefore present the weighted sum merely as a convenient default, not a methodological necessity.

\subsection{Linear Scalarization vs.\ Orthogonal Decomposition}
\label{app:linear_vs_ods}

Table~\ref{tab:linear_vs_ods} compares the ODS two-stage decomposition against a standard linear scalarization using identical per-sample scores (Hardness, CTMI, K-Center distance). ODS outperforms linear scalarization by $+6.5\%$ on the MTL Composite. The linear scheme suffers from \emph{diversity dilution}: highly diverse but uninformative outliers dominate the selection, degrading dense-prediction tasks severely.

\begin{table}[ht]
\centering
\caption{\textbf{Linear scalarization vs.\ ODS, identical scores.}}
\label{tab:linear_vs_ods}
\small
\begin{tabular}{lcccc}
\toprule
\textbf{Combination strategy} & \textbf{Cls Acc} & \textbf{Det Acc} & \textbf{Seg Dice} & \textbf{MTL Comp.} \\
\midrule
Linear: $(1\!-\!\delta)(\alpha H{+}\gamma\mathrm{CTMI})+\delta\,\mathrm{Div}$ & 38.8 & 24.00 & 11.24 & 0.04 \\
\rowcolor{gray!15}\textbf{ODS} (two-stage) & \textbf{42.14} & \textbf{26.52} & \textbf{17.44} & \textbf{0.105} \\
\midrule
$\Delta$ & $+3.34$ & $+2.52$ & $\mathbf{+6.2}$ & $+0.065$ \\
\bottomrule
\end{tabular}
\end{table}

\subsection{Hyperparameter Sensitivity (Pool Factor \texorpdfstring{$m$}{m} and MC Samples \texorpdfstring{$K$}{K})}
\label{app:sensitivity}

\textbf{Pool Factor ($m$):}
Table~\ref{tab:pool_sensitivity} demonstrates that ODS is robust across pool factors. Moderate values ($m=5$) achieve the optimal trade-off between filtering uninformative outliers and preserving sufficient manifold diversity, validating the Orthogonal Decomposition design.

\textbf{MC-Dropout Samples ($K$):}
Table~\ref{tab:mc_sensitivity} shows that $K=10$ provides a stable cross-task covariance estimate. Larger $K$ values yield only marginal gains while significantly increasing computational overhead, confirming CTMI's computational efficiency.

\begin{table}[ht]
\small
\centering
\begin{minipage}{0.45\textwidth}
    \centering
    \captionof{table}{\textbf{Sensitivity to Pool Factor $m$.}}
    \label{tab:pool_sensitivity}
    \setlength{\tabcolsep}{5pt}
    \begin{tabular}{lcccc}
    \toprule
    $m$ & 2 & 3 & 5 & 10 \\
    \midrule
    Score & 0.1032 & 0.1045 & \textbf{0.1051} & 0.1048 \\
    \bottomrule
    \end{tabular}
\end{minipage}
\hfill
\begin{minipage}{0.5\textwidth}
    \centering
    \captionof{table}{\textbf{Sensitivity to MC Samples $K$.}}
    \label{tab:mc_sensitivity}
    \begin{tabular}{lcccc}
    \toprule
    $K$ & 5 & 10 & 20 & 50 \\
    \midrule
    Score & 0.1041 & \textbf{0.1049} & 0.1044 & 0.1046 \\
    Time & 45 & 82 & 158 & 390 \\
    \bottomrule
    \end{tabular}
\end{minipage}
\end{table}

\subsection{Sensitivity to \texorpdfstring{$(\alpha,\gamma)$}{(alpha,gamma)} in the Stage-1 Score}
\label{app:alpha_gamma}

The Stage~1 informativeness score is $\mathcal{I}(x) = \alpha\,\mathrm{Hardness}(x) + \gamma\,\mathrm{CTMI}(x)$. As shown in Table~\ref{tab:alpha_gamma_full}, ODS is robust to a wide range of configurations. The performance surface is essentially flat near $(\alpha,\gamma) = (1,1)$, with no setting within a $4\times$ multiplier range losing more than $1.5\%$ relative MTL Composite. Severe performance drops occur only under degenerate configurations where either component is entirely removed ($\alpha=0$ or $\gamma=0$).

\begin{table}[ht]
\centering
\caption{\textbf{Sensitivity of ODS to $(\alpha,\gamma)$ on COCO.} Equal weighting ($\alpha=\gamma=1$) is near-optimal and robust to a $4\times$ swing in either weight.}
\label{tab:alpha_gamma_full}
\small
\setlength{\tabcolsep}{4pt}
\begin{tabular}{l|ccccccc}
\toprule
$(\alpha,\gamma)$ & $(0,1)$ & $(0.5,1)$ & $(1,0)$ & $(1,0.5)$ & $(1,1)$ & $(1,2)$ & $(2,1)$ \\
\midrule
MTL Comp. & 0.082 & 0.099 & 0.094 & 0.102 & \textbf{0.105} & 0.103 & 0.104 \\
\bottomrule
\end{tabular}
\end{table}


\section{Qualitative Analysis}
\label{app:part_qualitative}

\subsection{What the Conflict Ablation Shows}
\label{app:conflict_scope}

The controlled ablation in Table~\ref{tab:ablation_conflict} isolates the effect of the acquisition score: holding the optimizer (AdamW) and scoring pipeline fixed, penalizing gradient conflict lowers the MTL Composite from $0.076$ (ignoring conflict) to $0.059$ (penalizing it). This empirically grounds the \textit{Conflict Paradox}: penalizing conflict actively degrades data selection.

Crucially, this defines what the acquisition score must \emph{omit}, not what it must seek. ODS ranks candidates exclusively by Hardness and CTMI (Eq.~\eqref{eq:informativeness}); gradient conflict is absent from the score. This does not imply that ODS batches are inherently enriched with high-conflict samples relative to random selection. To understand the selection behavior, readers should refer to the component and combination-rule ablations (Appendices~\ref{app:component_ablation} and~\ref{app:combination_rule}), which isolate the contribution of each term without inferring a conflict-driven mechanism.

\subsection{Qualitative Analysis: Understanding Diversity Dilution}
\label{app:qualitative}

To visualize how ODS prevents \textit{Diversity Dilution}, we inspect the selected samples across active learning rounds along two axes: 
(1) \textbf{distributional coverage} via t-SNE embeddings (Figure~\ref{fig:tsne_analysis}), and 
(2) \textbf{visual and semantic content} (Figures~\ref{fig:qualitative_round1_2}--\ref{fig:qualitative_overlay_34}). 
In all qualitative figures, \textcolor{red}{\textbf{red boxes}} highlight redundant selections and characteristic failure modes of the baselines.

\subsubsection{Distributional Analysis: The Outlier Trap}

Figure~\ref{fig:tsne_analysis} tracks the t-SNE embeddings of selected samples against the unlabeled pool across Rounds 1--8. A stark behavioral divergence emerges:

\begin{itemize}
    \item \textbf{Redundant Collapse by Baselines.} Diversity-driven baselines (\textit{CoreSet}, \textit{BADGE}, \textit{BatchBALD}) repeatedly cluster in sparse, isolated embedding regions (e.g., $y < -40$ in Rounds 2, 3, and 5; red dashed circles). This illustrates the \textbf{inefficient over-sampling of visual outliers} that are distinctive but uninformative.
    
    \item \textbf{Controlled Exploration by ODS.} ODS also probes peripheral regions for coverage, but does so sparsely. By allocating the vast majority of its budget to the \textbf{high-density informative manifold}, ODS strictly balances outlier exploration with the exploitation of semantically rich data.
\end{itemize}

\subsubsection{Visual Redundancy vs.\ Semantic Diversity}
Figures~\ref{fig:qualitative_round1_2} and~\ref{fig:qualitative_round3_4} map these distributional failures to image space. \textbf{ODS} consistently acquires visually diverse, semantically dense scenes—from cluttered indoor environments to complex street views. Conversely, baselines exhibit severe redundancy, selecting visually similar samples despite operating on different acquisition principles. Geometry-based (\textit{CoreSet}) and information-theoretic (\textit{BatchBALD}) methods often converge on identical outliers, a redundancy mirrored in the t-SNE clusters (Figure~\ref{fig:tsne_analysis}, red dashed circles).

\paragraph{Failure Mode I: Semantic Sparsity (``Toilet Trap'').}
In Round 2 (Figure~\ref{fig:qualitative_round1_2}), several baselines fixate on semantically sparse environments. Highlighted by \textbf{red boxes}, \textit{CoreSet}, \textit{BADGE}, and \textit{BatchBALD} repeatedly acquire images of empty bathrooms, while \textit{Entropy} favors featureless kitchen counters. These samples are geometrically rare but contain minimal object interactions. This is \textit{Diversity Dilution} in action: prioritizing distance metrics collapses the selection onto samples with low semantic content.

\paragraph{Failure Mode II: Texture and Structural Bias.}
Baselines also over-sample high-frequency textures and repetitive structures. In Round 3 (Figure~\ref{fig:qualitative_round3_4}), \textit{CoreSet} disproportionately selects donuts, bananas, and pastries, while other baselines redundantly acquire images of fences or grids. While visually salient, these scenes offer negligible multi-task supervision compared to scenes with complex object interactions.

\subsubsection{Annotation Density and Multi-Task Supervision}

Figures~\ref{fig:qualitative_overlay_12} and~\ref{fig:qualitative_overlay_34} overlay ground-truth annotations, revealing stark density differences. ODS consistently targets scenes with dense bounding boxes and intricate segmentation masks (e.g., cluttered desks, crowded streets). These multi-instance scenes yield rich, high-frequency gradient signals that simultaneously improve detection and segmentation. In contrast, baselines frequently select sparse or empty images (e.g., ocean waves). These samples are ''hard`` only due to uninformative background uncertainty, yielding minimal information gain per labeled pixel.

\subsubsection{Why ODS Avoids Diversity Dilution}
Across both feature-space and image-space analyses, ODS systematically bypasses the failure modes that trap standard baselines. By filtering out uninformative outliers before maximizing coverage, ODS consistently acquires semantically complex, non-redundant data—discovering the high-value samples that existing MTL-AL strategies overlook.

\begin{figure}[ht!]
    \centering
    \includegraphics[width=\linewidth]{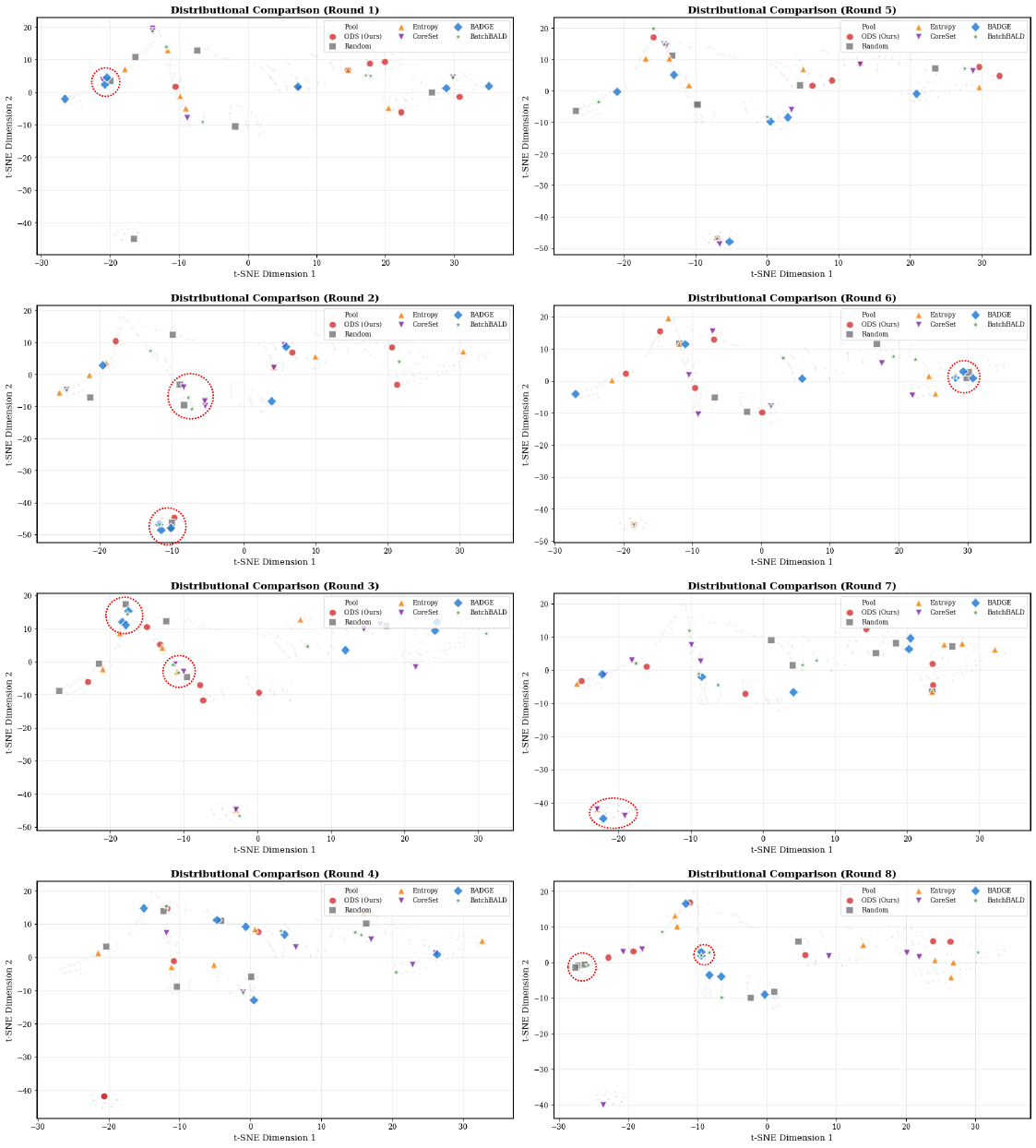}
    \caption{\textbf{t-SNE Distributional Analysis across Active Learning Rounds (1--8).}
    We visualize the embedding space of selected samples from different acquisition strategies, overlaid on the unlabeled pool (gray).
    \textbf{Red dashed circles highlight regions of redundant selection by baselines.}
    \textbf{(Baseline Inefficiency):} Diversity-driven methods such as \textit{CoreSet} and \textit{BADGE} repeatedly select samples concentrated within narrow, low-density regions, leading to \textbf{redundant over-sampling of visual outliers} (e.g., isolated bottom clusters in R2, R3, R5, and R7).
    \textbf{(Controlled Exploration by ODS):} Although \textbf{ODS} ($\color{red}\bullet$) also explores these peripheral regions to ensure coverage, it does so \textbf{sparsely}, avoiding the clustering behavior observed in baselines.
    \textbf{(Balanced Allocation):} As a result, ODS allocates most of its acquisition budget to the \textbf{high-density informative manifold}, achieving a more effective balance between exploration of outliers and exploitation of the main data distribution.}   
    \label{fig:tsne_analysis}
\end{figure}

\begin{figure*}[ht!]
    \centering
    \includegraphics[width=0.85\linewidth]{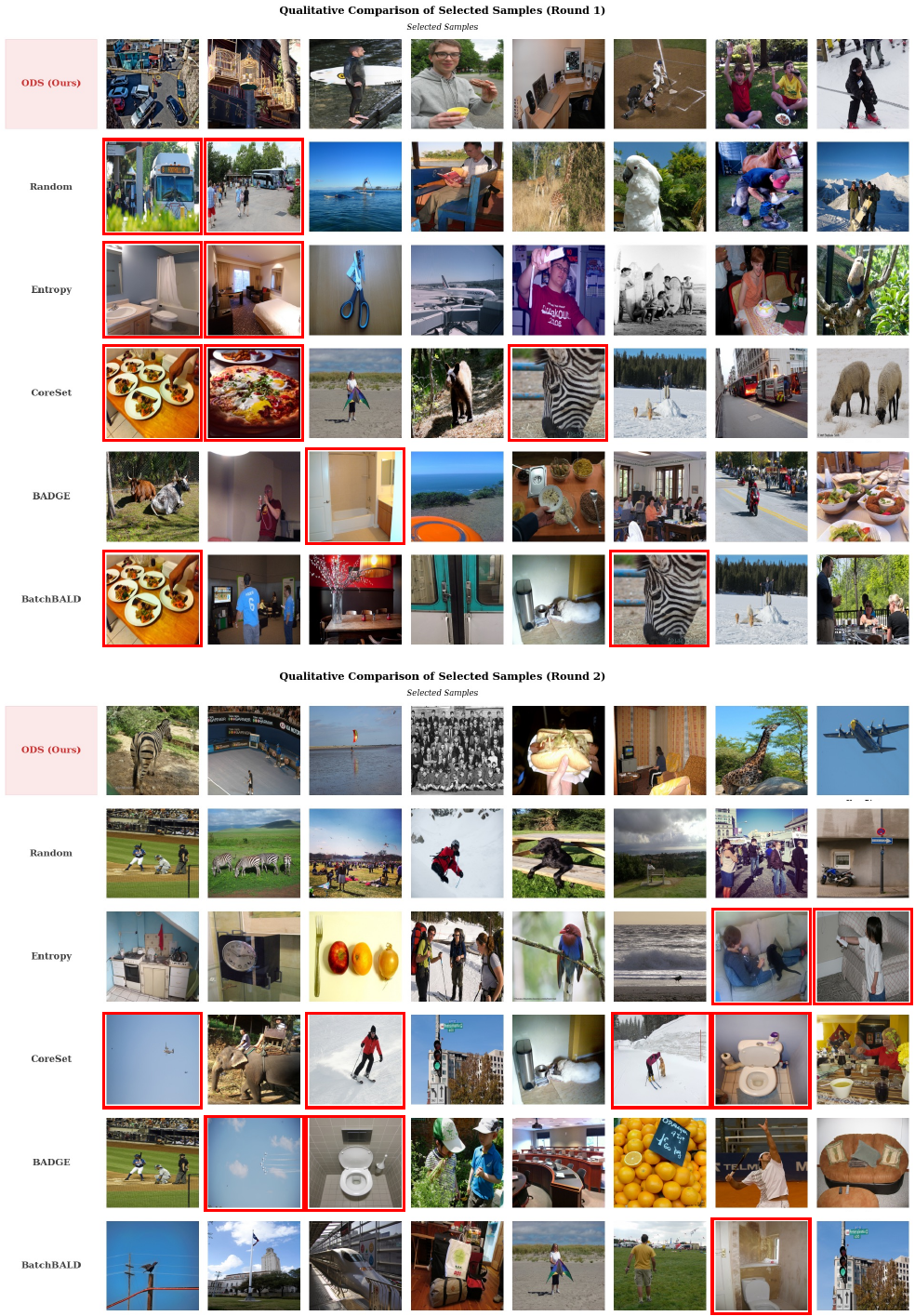}
    \caption{\textbf{Qualitative Comparison of Selected Samples (COCO, Rounds 1--2).}
    Each row shows images selected by a different acquisition strategy.
    \textbf{\textcolor{red}{Red boxes} mark redundant selections by baselines,} where visually or semantically similar samples are repeatedly acquired within the same round.
    In contrast, \textbf{ODS avoids such early-stage redundancy}, allocating its budget to more diverse and complementary visual content.}
    \label{fig:qualitative_round1_2}
\end{figure*}

\begin{figure*}[h!]
    \centering
    \includegraphics[width=0.85\linewidth]{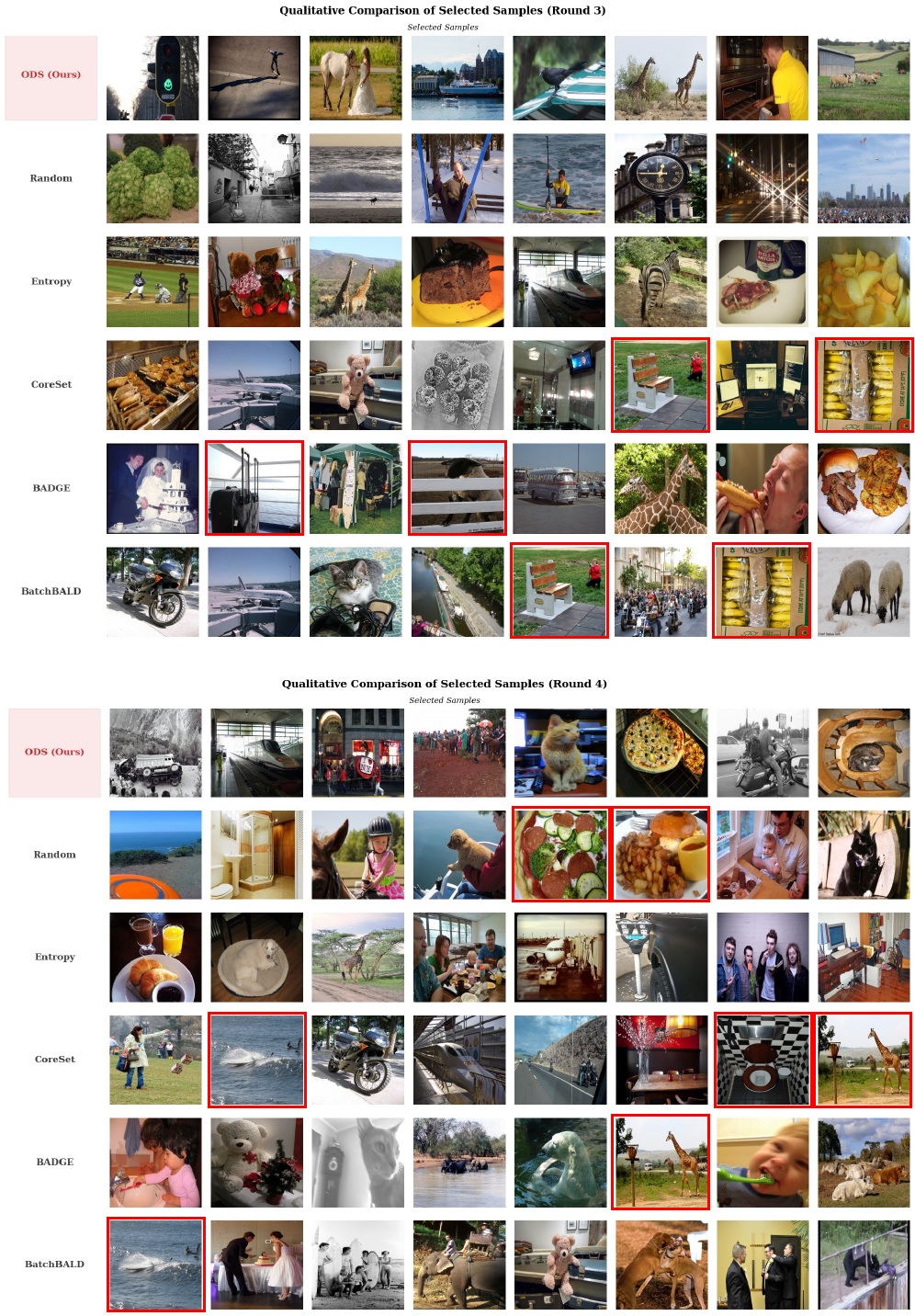}
    \caption{\textbf{Qualitative Comparison of Selected Samples (COCO, Rounds 3--4).}
    \textbf{\textcolor{red}{Red boxes} highlight structural and textural redundancies} among samples selected by baseline methods, such as repeated acquisition of near-identical scenes or objects.
    Despite operating at later rounds with a more informative pool, baselines continue to exhibit inefficient budget usage, whereas \textbf{ODS maintains diverse coverage without collapsing onto narrow visual patterns}.}
    \label{fig:qualitative_round3_4}
\end{figure*}

\begin{figure*}[h!]
    \centering
    \includegraphics[width=0.95\linewidth]{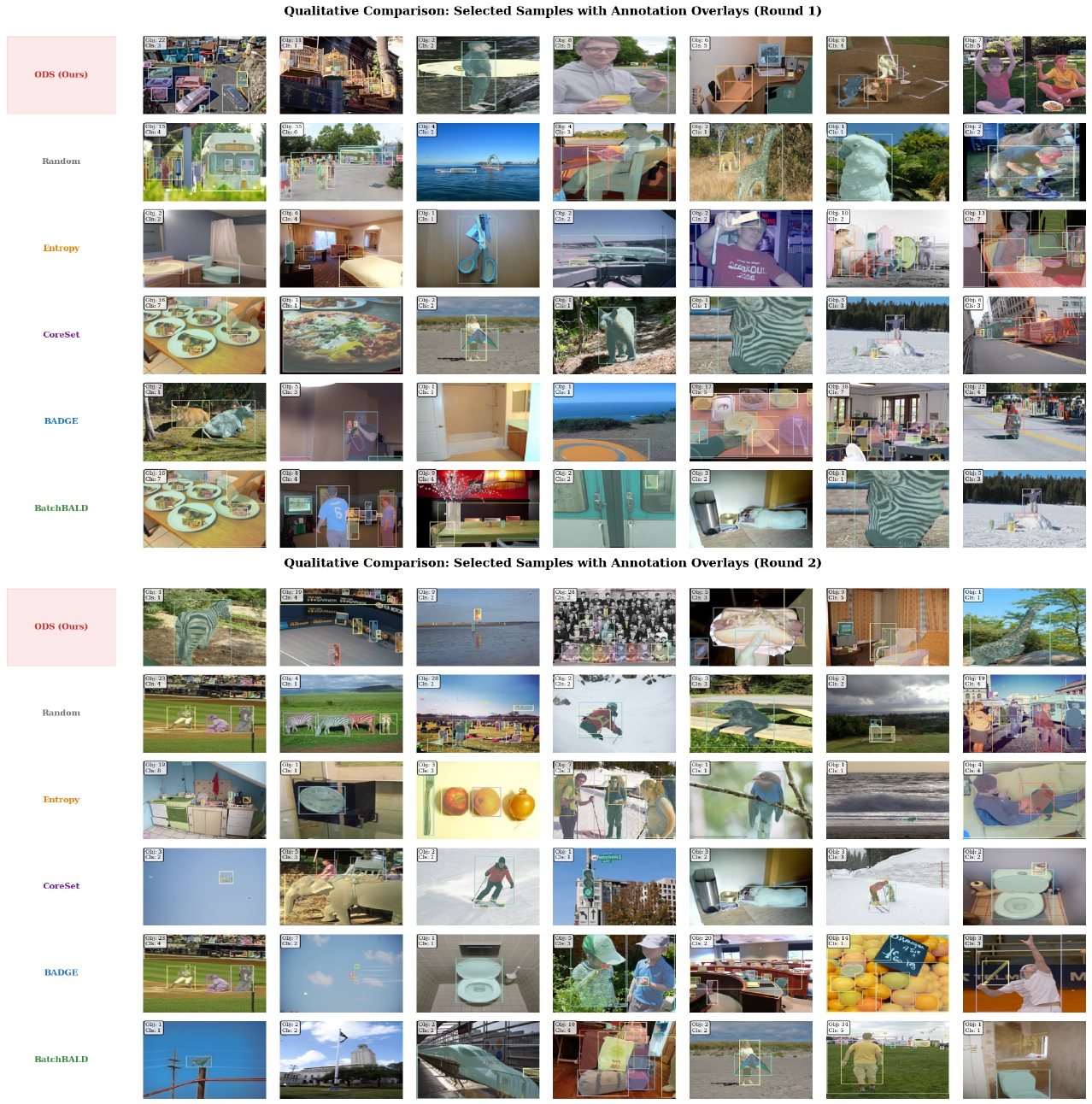}
    \caption{\textbf{Annotation Density Analysis (COCO, Rounds 1 \& 2).}
    We visualize the ground-truth bounding boxes and segmentation masks for selected samples.
    \textbf{ODS (Top Row):} consistently selects images with \textbf{high instance density} (e.g., the cluttered desk in R1, the crowd in R2), providing rich supervision for multi-task learning.
    \textbf{Baselines (Bottom Rows):} frequent failure modes include selecting ``empty'' images with \textbf{sparse or no annotations} (e.g., \textit{Entropy} selecting less informative rooms in R2, \textit{Random} selecting simple single-object images). This confirms that ODS maximizes information gain per pixel.}
    \label{fig:qualitative_overlay_12}
\end{figure*}

\begin{figure*}[h!]
    \centering
    \includegraphics[width=0.95\linewidth]{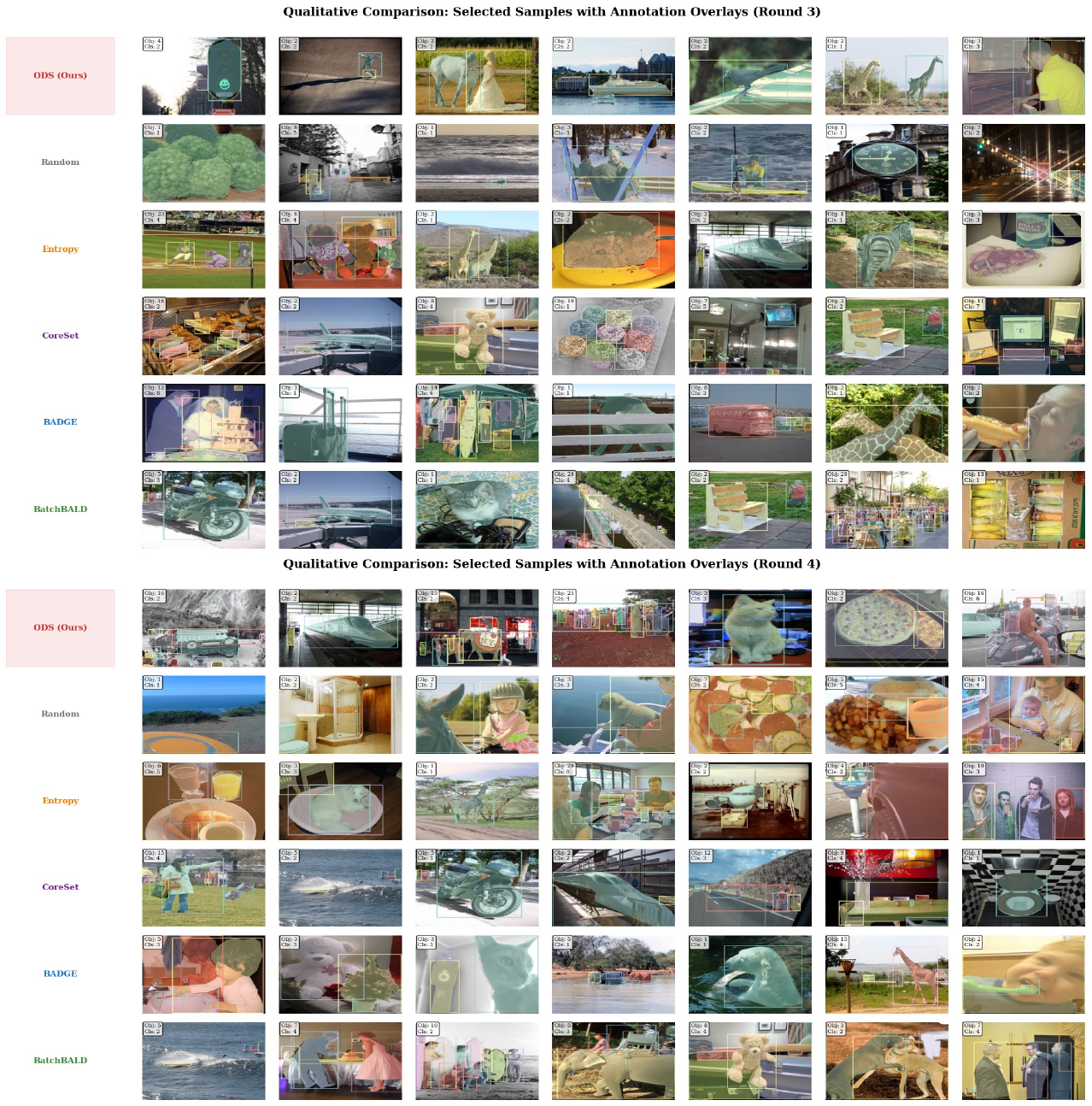}
    \caption{\textbf{Annotation Density Analysis (COCO, Rounds 3 \& 4).}
    Overlays reveal the semantic complexity of selected samples.
    \textbf{ODS (Top Rows):} selects structurally complex scenes with overlapping objects, forcing the model to resolve difficult occlusions.
    \textbf{Baselines:} \textit{CoreSet} and \textit{BatchBALD} often select images with large, simple textures or single large objects, which offer limited geometric challenges compared to the dense scenes selected by ODS.}
    \label{fig:qualitative_overlay_34}
\end{figure*}


\clearpage
\section{Ethics statement}
\vspace{-5pt}
This work investigates acquisition strategies for multi-task learning and collects no new data. All experiments rely on the public COCO, Cityscapes, and NYUv2 benchmarks under their respective licenses. No human-subject research was conducted, and we handle no personally identifiable information beyond the anonymized distributions provided by the dataset creators (e.g., blurred faces and license plates in Cityscapes). Like any active learning method, our acquisition function concentrates the annotation budget on a specific data subset, which may amplify existing pool imbalances. Because our score relies strictly on predictive uncertainty and cross-task dependency without explicit fairness constraints, deployment in safety-critical settings requires auditing the selected batches. We anticipate no risks of misuse beyond those inherent to general label-efficient learning.

\section{Reproducibility statement}
\vspace{-5pt}
All code, configuration files, and analysis scripts required to reproduce our results are publicly available at \href{https://github.com/Juho-Jung/multitask_active_learning_official}{\texttt{github.com/Juho-Jung/multitask\_active\_learning\_official}}. Appendix~\ref{app:implementation} specifies the architectures, optimization settings, and MC-dropout configurations. Table~\ref{tab:protocol_settings} lists the exact $(m, K, \alpha, \gamma)$ hyperparameters for all experimental protocols. Appendix~\ref{app:datasets} documents the dataset splits and seed sets, and Appendix~\ref{app:algorithm} details the complete selection procedure. Unless noted otherwise, results are reported across three random seeds ($42$, $123$, $456$) as mean $\pm$ standard deviation. While non-deterministic GPU operations may cause minor per-round numerical variances, the performance gaps and rank orderings central to our claims are strictly robust and fully reproducible.

\section{Use of AI}
\vspace{-5pt}
Large language models served exclusively as writing assistants for grammar, \LaTeX{} formatting, and literature searches. They contributed no original research ideas, methodological designs, experimental executions, or data analyses. All technical claims, mathematical proofs, and empirical results were produced directly by the human authors, with LLMs limited to secondary proofreading. The authors take full responsibility for the content of this paper. We provide this declaration for transparency, though it falls under the permitted-use category of the NeurIPS LLM policy.

\clearpage
\section*{NeurIPS Paper Checklist}

\begin{enumerate}

\item {\bf Claims}
    \item[] Question: Do the main claims made in the abstract and introduction accurately reflect the paper's contributions and scope?
    \item[] Answer: \answerYes{}
    \item[] Justification: The abstract and introduction (Section~\ref{sec:intro}) state our three contributions --- the Conflict Paradox, Orthogonal Decomposition, and CTMI --- which match the formal results (Proposition~\ref{thm:conflict_paradox}, Assumption~\ref{assum:orthogonality}, Proposition~\ref{prop:attribution}, Algorithm~\ref{alg:ods}) and the empirical evidence (Tables~\ref{tab:main_results}, \ref{tab:ablation_conflict}, \ref{tab:linear_vs_ods}, \ref{tab:robustness_detail}, \ref{tab:pcgrad_ablation}). Empirical scope (COCO, Cityscapes, NYUv2; lightweight scratch encoders by default) is stated up front.

\item {\bf Limitations}
    \item[] Question: Does the paper discuss the limitations of the work performed by the authors?
    \item[] Answer: \answerYes{}
    \item[] Justification: We include a dedicated Limitations section (Section~\ref{sec:limitations}) covering: (L1) the sample-level / objective-level distinction in Proposition~\ref{thm:conflict_paradox}; (L2) covariance-based CTMI being blind to non-linear dependence (mitigated by maximum-entropy lower-bound semantics and validated by distance correlation in Appendix~\ref{app:dcorr}); (L3) the $(m, K, \alpha, \gamma)$ hyperparameters (with sensitivity studies); (L4) the lightweight scratch backbone scope (with a pretrained ResNet-50 study in Appendix~\ref{app:pretrained_ms}).

\item {\bf Theory assumptions and proofs}
    \item[] Question: For each theoretical result, does the paper provide the full set of assumptions and a complete (and correct) proof?
    \item[] Answer: \answerYes{}
    \item[] Justification: All theoretical results state their assumptions explicitly and reference complete proofs. Proposition~\ref{thm:conflict_paradox} (sample-level conflict structure) lists three assumptions and the full proof appears in Appendix~\ref{app:proof_paradox}. Proposition~\ref{prop:attribution} (shared-representation attribution for CTMI) is proved in Appendix~\ref{app:proof_attribution}; the submitted version's maximum-entropy lower-bound claim was withdrawn during discussion and the appendix states explicitly what replaced it and why. The two-stage approximation guarantee (Proposition in Appendix~\ref{app:approx_guarantee}) states the empirical-independence assumption and its $2$-approximation component (Sener \& Savarese). Proof sketches in the main text are paired with formal proofs in the appendix; we are explicit about heuristic vs.\ formal claims (see Remark on sample-level vs.\ objective-level Pareto stationarity).

    \item {\bf Experimental result reproducibility}
    \item[] Question: Does the paper fully disclose all the information needed to reproduce the main experimental results of the paper to the extent that it affects the main claims and/or conclusions of the paper (regardless of whether the code and data are provided or not)?
    \item[] Answer: \answerYes{}
    \item[] Justification: Appendix~\ref{app:implementation} reports the model architectures, the active-learning protocol (initial labeled, budget per round, number of rounds), the optimizer (PCGrad with AdamW, learning rate $10^{-3}$, cosine schedule), batch size, regularization, and dropout rate. ODS-specific hyperparameters $(m, K, \alpha, \gamma)$ are listed per protocol in Table~\ref{tab:protocol_settings}, and Algorithm~\ref{alg:ods} is fully specified. Datasets are public and described in Appendix~\ref{app:datasets}. The code repository is linked in the abstract. 

\item {\bf Open access to data and code}
    \item[] Question: Does the paper provide open access to the data and code, with sufficient instructions to faithfully reproduce the main experimental results, as described in supplemental material?
    \item[] Answer: \answerYes{}
    \item[] Justification: We release the full implementation at the anonymous repository linked in the abstract, including training/AL loops, all baselines, and configuration files for COCO, Cityscapes, and NYUv2. Datasets are public; access procedures for COCO~\cite{lin2014microsoft}, Cityscapes~\cite{cordts2016cityscapes}, and NYUv2~\cite{silberman2012indoor} follow the standard splits and licensing terms documented in Appendix~\ref{app:datasets}.

\item {\bf Experimental setting/details}
    \item[] Question: Does the paper specify all the training and test details (e.g., data splits, hyperparameters, how they were chosen, type of optimizer) necessary to understand the results?
    \item[] Answer: \answerYes{}
    \item[] Justification: Section~\ref{sec:experiments} provides the high-level setup (datasets, baselines, fairness controls, the MTL Composite Score and Average Rank metrics). Appendix~\ref{app:implementation} adds full architectures, training hyperparameters (LR, batch size, optimizer, scheduler), the AL protocol (initial labeled $=1000$, $30$ rounds, batch $=32$), MC-Dropout settings ($K=20$, $p=0.05$), and how $(\alpha,\gamma)$ are chosen with adaptive scheduling. Sensitivity studies are reported in Appendices~\ref{app:sensitivity} and~\ref{app:alpha_gamma}.

\item {\bf Experiment statistical significance}
    \item[] Question: Does the paper report error bars suitably and correctly defined or other appropriate information about the statistical significance of the experiments?
    \item[] Answer: \answerYes{}
    \item[] Justification: Main Table~\ref{tab:main_results} reports mean$\pm$std over $3$ seeds $\{42, 123, 456\}$ for all methods on COCO under identical protocols. The pretrained-backbone study (Appendix~\ref{app:pretrained_ms}) is also reported with 3-seed mean$\pm$std. We discuss the statistical-significance limitations of 3 seeds in Section~\ref{sec:limitations} and rely on cross-task / cross-dataset rank consistency as complementary evidence.    

\item {\bf Experiments compute resources}
    \item[] Question: For each experiment, does the paper provide sufficient information on the computer resources (type of compute workers, memory, time of execution) needed to reproduce the experiments?
    \item[] Answer: \answerYes{}
    \item[] Justification: Appendix~\ref{app:implementation} states that experiments were conducted on $8\times$ NVIDIA V100 (32GB) GPUs. Appendix~\ref{app:runtime} reports per-round wall-clock acquisition time (mean and median in seconds) for ODS and all baselines, isolating selection cost from training cost.

\item {\bf Code of ethics}
    \item[] Question: Does the research conducted in the paper conform, in every respect, with the NeurIPS Code of Ethics \url{https://neurips.cc/public/EthicsGuidelines}?
    \item[] Answer: \answerYes{}
    \item[] Justification: The research uses only publicly released, standard computer-vision benchmarks (COCO, Cityscapes, NYUv2), proposes an active-learning algorithm, and involves no human subjects, sensitive personal data, or deployment scenarios. We have reviewed the NeurIPS Code of Ethics and the work conforms in every respect.

\item {\bf Broader impacts}
    \item[] Question: Does the paper discuss both potential positive societal impacts and negative societal impacts of the work performed?
    \item[] Answer: \answerYes{}
    \item[] Justification: The proposed method reduces the labeled-data requirement for multi-task vision systems, with a positive impact on settings where annotation is expensive (e.g., medical, autonomous driving prototyping). We do not foresee a direct negative societal impact specific to this work; the underlying tasks (classification, detection, segmentation, depth estimation) are general-purpose and not tied to a high-risk deployment. Like any data-efficient learning method, ODS could in principle reduce the cost of training models for downstream applications with negative impact, but the algorithmic contribution is upstream of any specific deployment.

\item {\bf Safeguards}
    \item[] Question: Does the paper describe safeguards that have been put in place for responsible release of data or models that have a high risk for misuse (e.g., pre-trained language models, image generators, or scraped datasets)?
    \item[] Answer: \answerNA{}
    \item[] Justification: The paper does not release a generative model, scraped dataset, or any artifact with a high risk of misuse. We release only an active-learning algorithm and lightweight task models trained on standard public datasets (COCO, Cityscapes, NYUv2).

\item {\bf Licenses for existing assets}
    \item[] Question: Are the creators or original owners of assets (e.g., code, data, models), used in the paper, properly credited and are the license and terms of use explicitly mentioned and properly respected?
    \item[] Answer: \answerYes{}
    \item[] Justification: All datasets used in the paper are publicly available and are cited at first use: COCO~\cite{lin2014microsoft} (CC BY 4.0 / Flickr terms for images), Cityscapes~\cite{cordts2016cityscapes} (research-only license), NYUv2~\cite{silberman2012indoor} (research-only license). Baseline implementations follow the publicly released reference implementations of BADGE~\cite{ash2019deep}, BatchBALD~\cite{kirsch2019batchbald}, CoreSet~\cite{sener2017active}, PCGrad~\cite{yu2020gradient}, Hekimoglu et al.~\cite{hekimoglu2024active}, and PartAL~\cite{durasov2022partal}, all under their original licenses.

\item {\bf New assets}
    \item[] Question: Are new assets introduced in the paper well documented and is the documentation provided alongside the assets?
    \item[] Answer: \answerYes{}
    \item[] Justification: The only new asset is the source code for ODS, released at the anonymous repository linked from the abstract. The repository includes a README, configuration files, training/AL loops, and scripts to reproduce the headline results on each dataset. No new datasets or pre-trained models are released.

\item {\bf Crowdsourcing and research with human subjects}
    \item[] Question: For crowdsourcing experiments and research with human subjects, does the paper include the full text of instructions given to participants and screenshots, if applicable, as well as details about compensation (if any)?
    \item[] Answer: \answerNA{}
    \item[] Justification: The paper does not involve crowdsourcing or research with human subjects. All experiments are conducted on existing publicly released image datasets.

\item {\bf Institutional review board (IRB) approvals or equivalent for research with human subjects}
    \item[] Question: Does the paper describe potential risks incurred by study participants, whether such risks were disclosed to the subjects, and whether Institutional Review Board (IRB) approvals (or an equivalent approval/review based on the requirements of your country or institution) were obtained?
    \item[] Answer: \answerNA{}
    \item[] Justification: The paper does not involve human subjects, so no IRB approval is required.

\item {\bf Declaration of LLM usage}
    \item[] Question: Does the paper describe the usage of LLMs if it is an important, original, or non-standard component of the core methods in this research? Note that if the LLM is used only for writing, editing, or formatting purposes and does \emph{not} impact the core methodology, scientific rigor, or originality of the research, declaration is not required.
    \item[] Answer: \answerNA{}
    \item[] Justification: LLMs were not used as a component of the core methodology, scientific results, or experimental design. They may have been used for incidental writing/editing assistance, which per NeurIPS policy does not require declaration.

\end{enumerate}

\end{document}